\documentclass{article}

\usepackage{microtype}
\usepackage{graphicx}
\usepackage{caption}
\usepackage{subcaption}
\usepackage{booktabs} 
\usepackage{tikz}
\usepackage{import}
\usepackage{bm}
\usepackage{tabularx}

\usepackage{hyperref}

\usepackage[preprint]{icml2026}

\usepackage{amsmath}
\usepackage{amssymb}
\usepackage{mathtools}
\usepackage{amsthm}
\usepackage{subcaption}

\DeclareMathOperator{\Tr}{Tr}

\newcommand{\out}{\text{out}}
\newcommand{\softmax}{\operatorname{softmax}}

\newcommand{\din}{d_{\text{in}}}
\DeclareMathOperator*{\argmin}{\operatorname{argmin}}
\DeclareMathOperator*{\E}{\mathbb{E}}

\DeclareMathOperator*{\tr}{\operatorname{tr}}
\newcommand{\JL}{\text{JL}}
\newcommand{\reg}{\text{reg}}

\newcommand{\cB}{\mathcal{B}}
\newcommand{\cD}{\mathcal{D}}
\newcommand{\cE}{\mathcal{E}}
\newcommand{\cF}{\mathcal{F}}
\newcommand{\cG}{\mathcal{G}}
\newcommand{\cH}{\mathcal{H}}
\newcommand{\cI}{\mathcal{I}}
\newcommand{\cL}{\mathcal{L}}
\newcommand{\cP}{\mathcal{P}}
\newcommand{\cR}{\mathcal{R}}
\newcommand{\cS}{\mathcal{S}}
\newcommand{\cT}{\mathcal{T}}
\newcommand{\cU}{\mathcal{U}}
\newcommand{\cX}{\mathcal{X}}
\newcommand{\cY}{\mathcal{Y}}
\newcommand{\e}{\varepsilon}
\newcommand{\R}{\mathbb{R}}

\newcommand{\tD}{\tilde{\mathcal{D}}}
\newcommand{\tX}{\tilde{X}}

\newcommand{\tY}{\tilde{Y}}

\newcommand{\tf}{\tilde{f}}
\newcommand{\tPhi}{\tilde{\Phi}}

\newcommand{\tK}{\tilde{K}}
\newcommand{\vtx}{\mathbf{\tilde{x}}}
\newcommand{\vty}{\mathbf{\tilde{y}}}

\newcommand{\tlambda}{\tilde{\lambda}}
\newcommand{\tdelta}{\tilde{\delta}}

\newcommand{\STATEX}{\item[]}

\newcommand{\valpha}{\boldsymbol{\alpha}}
\newcommand{\vu}{\mathbf{u}}
\newcommand{\vv}{\mathbf{v}}
\newcommand{\vx}{\mathbf{x}}
\newcommand{\vy}{\mathbf{y}}
\newcommand{\vz}{\mathbf{z}}

\usepackage[capitalize,noabbrev]{cleveref}

\theoremstyle{plain}
\newtheorem{theorem}{Theorem}[section]
\newtheorem{proposition}[theorem]{Proposition}
\newtheorem{lemma}[theorem]{Lemma}
\newtheorem{corollary}[theorem]{Corollary}
\theoremstyle{definition}
\newtheorem{definition}[theorem]{Definition}

\theoremstyle{remark}
\newtheorem{remark}[theorem]{Remark}

\usepackage[textsize=tiny]{todonotes}

\usepackage{marginnote}
\definecolor{burntorange}{rgb}{0.8, 0.33, 0.0}

\icmltitlerunning{Efficient Analysis of the Distilled Neural Tangent Kernel}

\begin{document}
\twocolumn[
  \icmltitle{Efficient Analysis of the Distilled Neural Tangent Kernel}

  \icmlsetsymbol{equal}{*}

  \begin{icmlauthorlist}
    \icmlauthor{Jamie Mahowald}{lanl}
    \icmlauthor{Brian Bell}{lanl}
    \icmlauthor{Alex Ho}{lanl}
    \icmlauthor{Michael Geyer}{lanl}
  \end{icmlauthorlist}

  \icmlaffiliation{lanl}{Los Alamos National Laboratory, Los Alamos, NM}

  \icmlcorrespondingauthor{Jamie Mahowald}{j.mahowald@lanl.gov}

  \icmlkeywords{Kernel methods, neural tangent kernel, compression}

  \vskip 0.3in
]


\printAffiliationsAndNotice{}  

\begin{abstract}
Neural tangent kernel (NTK) methods are computationally limited by the need to evaluate large Jacobians across many data points. Existing approaches reduce this cost primarily through projecting and sketching the Jacobian. We show that NTK computation can also be reduced by compressing the \textit{data dimension itself} using NTK-tuned dataset distillation. We demonstrate that the neural tangent space spanned by the input data can be induced by dataset distillation, yielding a 20-100\(\times\) reduction in required Jacobian calculations. We further show that per-class NTK matrices have low effective rank that is preserved by this reduction. Building on these insights, we propose the \textbf{distilled neural tangent kernel} (\textbf{DNTK}), which combines NTK-tuned dataset distillation with state-of-the-art projection methods to reduce up NTK computational complexity by up to five orders of magnitude while preserving kernel structure and predictive performance. 

\end{abstract}
\section{Introduction}

The neural tangent kernel (NTK) \cite{jacot_ntk_2018} gives a theoretical lens for understanding neural network (NN) training, particularly in the overparameterized regime. As the width of a given NN approaches infinite, the network's training dynamics become linear and equivalent to kernel regression. In this ``lazy training'' regime, the network adjusts a linear combination of fixed features, and does not learn to represent features from scratch. Under certain circumstances, a model can be approximated within the kernel regime defined by the NTK, which enables analysis that would be intractable in the parametric regime.

Unfortunately, the NTK is intractable to compute for all but the smallest networks: for a parameterized network \(f\) with \(P\) fixed total parameters trained on \(n\) data points, the cost of computing the NTK scales as \(O(n^2 P)\), while storing scales as \(O(n^2)\). To naïvely compute the NTK of a ResNet50 on all 1.3 million ImageNet points requires at least \(4.2 \cdot 10^{19}\) floating point operations and \(1.69 \cdot 10^{12}\) memory entries. Projects that leverage NTK-like formulations to detect distributional shifts \cite{huang_distshift_2021}, quantify uncertainty \cite{wilson_uq_2025}, and characterize robustness \cite{tsilivis_adv_2022} have therefore been limited to small models. To enable these analyses for larger models, we need a robust NTK approximation regime.

Prior attempts to approximate the NTK for downstream tasks tend to either address parameter complexity alone \cite{hirsch_pinn_2025} or approximate it using different, cheaper kernels \cite{loo_rfad_2022}. Our method creates tractable, accurate approximations of the NTK itself by noting that the empirical NTK of a pretrained neural network exhibits significant redundancy at \textit{three} levels: in the dataset, in the parameters, and in the gradient subspace. To overcome the intractability of computing full NTKs, we combine three complementary strategies that target redundancy at each of these levels: (1) data distillation, which synthesizes compact datasets that preserve task performance; (2) random projection, which reduces the dimensionality of the tangent space while preserving kernel structure; and (3) structure-aware gradient distillation, designed to further compress the NTK by exploiting its local and global spectral structure. 

Central to our approach is that these components – dataset distillation, random projection, and gradient distillation – are theoretically justified methods that target distinct sources of redundancy. To that end, we provide proofs in \Cref{sec:setup} and in \Cref{appsec:error_bounds,appsec:pca_proof} that explain how and when these methods preserve variances and subspaces.
We refer to NTK approximations constructed via this unified framework as \textbf{distilled neural tangent kernels} (\textbf{DNTKs}). Combining these techniques, we achieve up to a \(\approx 10^5\times\) reduction in both computation time and storage in our experiments on a mid-size image classification task while maintaining downstream task performance with theoretical guarantees on approximation quality.


The paper is organized as follows: \Cref{sec:related} situates our work within NTK theory and dataset distillation. \Cref{sec:setup} formalizes the empirical NTK and notions of redundancy that motivate our approach. \Cref{sec:method} describes our compression pipeline: dataset distillation, random projection, and gradient distillation. \Cref{sec:experiments} presents experimental validation.
\section{Related work}\label{sec:related}

This work draws on two lines of research: (i) kernel perspectives on deep networks, especially the \textbf{neural tangent kernel (NTK)}, and (ii) \textbf{dataset distillation} / coreset methods for summarizing data.

\paragraph{Kernel methods and deep kernel learning.}
A long line of work connects neural networks to kernel machines. In particular, deep kernel learning combines hierarchical representations with the nonparametric flexibility of kernels \citep{wilson_kernel_2016,huang_hierarchy_2023}, and admits representer-theorem-style characterizations for composed RKHS models \citep{bohn_representer_2019}.

\paragraph{Neural tangent kernel (NTK).} 
\citet{jacot_ntk_2018} first proposed the NTK in 2018 to explain training dynamics in the infinite-width limit. Later work on making NTK computation tractable has focused on sketching and random-feature approximations to obtain fast (near input-sparsity) approximations of NTK matrices. For example, \citet{zandieh_sketching_2021} sketch polynomial expansions of arc-cosine kernels and combine sketching with random features to obtain spectral approximations, while \citet{han_activations_2022} generalize these ideas beyond ReLU via truncated Hermite expansions for broad activation classes.
\citet{hirsch_pinn_2025} develop a random-sketching-based approach with a physics-informed loss function. 
These methods demonstrate the effectiveness of random projection for reducing parameter complexity. 

In contrast, our work targets redundancy on the \emph{data} side (via distillation) as a complement to sketching-based accelerations of kernel construction.
To diagnose when such kernel approximations are effective, recent work studies NTK spectra: \citet{lin-feature-2025} uses empirical-NTK eigenanalysis to surface learned features, while \citet{benigni-eigen-2025} characterizes limiting NTK eigenvalue distributions under high-dimensional scaling.

\paragraph{Dataset distillation.} First proposed by  \citet{wang_distillation_2018}, dataset distillation has inspired numerous offshoots that aim both to streamline the core algorithm and to apply it to various use cases. Several kernel-based distillation methods cast distillation in a kernel ridge-regression objective inspired by infinite-width neural kernels \citep{nguyen_kip_2021}, and accelerate it via random-feature approximations \citep{loo_rfad_2022}. Whereas these works use neural kernels to formulate and accelerate distillation, we use distillation to accelerate downstream NTK computations.
We employ a modified version of WMDD \citep{liu_wmdd_2025}, which distills data by Wasserstein-metric feature matching (via a Wasserstein barycenter) in a pretrained feature space.

\paragraph{Coresets and sampling.}
Our local/global selection procedure (\Cref{alg:local_global_comp}) is related to coreset frameworks for clustering and shape fitting \citep{feldman-clustering-2016}. Our kernel sketching strategy also connects to random-feature approximations for scaling kernel machines \citep{rahimi-kernel-2007}.
\section{Setup}\label{sec:setup}

Let \(\theta \in \R^P\) denote the \(P\)-dimensional parameters of a neural network
\(f(\vx;\theta):\R^{d_{\mathrm{in}}}\to\R^C\), mapping inputs to class logits.

\subsection{NTK and KRR computation}\label{sec:krr-complexity}
The \textbf{neural tangent kernel} (NTK) measures gradient alignment between inputs. Letting \(\phi(\vx) = \nabla_\theta f(\vx; \theta) \in \R^{C \times P}\),
\begin{equation}\label{eqn:ntk-simple}
K(\vx,\vx')= \Tr(\phi(\vx)^\top \phi(\vx')) = \phi \phi^\top.
\end{equation}
For classification, we work with per-class kernels \(K^c\) formed from gradients
\(\phi^c(\vx):=\nabla_\theta f^c(\vx;\theta)\in\R^P\) of each logit (in this setting, \(K^c(\vx, \vx') = \phi^c(\vx) \phi^c(\vx')\) is a scalar). Given a training set \(X\) of \(n\) training points,
the \textbf{class gradient matrix} \(\Phi^c\in\R^{n\times P}\) has rows \([\Phi^c]_i=\phi^c(x_i)\),
yielding the \textbf{class kernel} \(K^c=\Phi^c(\Phi^c)^\top\in\R^{n\times n}\).

Under certain conditions (the network operates near the lazy training regime where dynamics are approximately linear in parameters, \(K_{XX}\) eigenvalues decay rapidly, and training labels align with dominant eigendirections, see \Cref{appsec:error_bounds}), we can approximate \(f^c\) by a kernel representer \(f_K^c\) obtained via kernel ridge regression (KRR, \Cref{appsec:ridge-regression}): \(f_K^c\) is fit on \(X\) and evaluated on a test set \(X^*\) of size \(n_{\mathrm{test}}\). The memory required to materialize \(K^c\) and \(K^c_{XX^*}\) is \(O(nP + n^2 + nn_{\rm test})\), which is prohibitive at scale for all \(C\) classes.
Alternatively, if we store only gradients and compute kernel entries on the fly, memory is \(O((n+n_{\rm test})PC)\), but computation remains dominated by \(P\)-dimensional inner products.

To address these complexity concerns, we introduce and exploit notions of redundancy in several spaces.

\subsection{Redundancy in data and parameters}\label{sec:redundancy}
By \textbf{redundancy}, we mean that most variation in model training dynamics and predictions can be
explained in a significantly smaller subspace than the model uses in practice.\footnote{Many of the notions
surrounding this idea are based on the \textbf{manifold hypothesis}; see \citet{fefferman_manifold_2016}.}
Consider a fixed parameter vector \(\theta\) of dimension \(P\), a labeled dataset \(\cD=(X,Y)\), and class \(c\). Let
\(K_{XX}^c\) be the NTK class kernel matrix at \(\theta\) between \(X\) and itself (that is, the matrix \(K^c\) where \(K_{ij}^c = K^c(\vx_i, \vx_j)\) for \(\vx_i\) and \(\vx_j \in X\)). We focus on two kinds of
redundancy, defined via Gram matrices on interrelated spaces.

\begin{definition}[Data redundancy]\label{def:data-redundancy}
Let \(\lambda_1 \ge \cdots \ge \lambda_n\) denote the eigenvalues of \(K_{XX}^c\) in decreasing order. The
\textbf{truncation rank} of \(K_{XX}^c\) at threshold \(\e\) is
\[
r_{\mathrm{trunc}}(K_{XX}^c,\e)
:=\min\left\{k:\frac{\sum_{i=1}^k\lambda_i}{\sum_{i=1}^n\lambda_i}\ge 1-\e\right\}.
\]
The input set \(X\) is \textbf{\((r,\e)\)-data-redundant} if
\(r_{\mathrm{trunc}}(K_{XX}^c,\e) \le n/r\).
\end{definition}
In other words, an \((r,\e)\)-data-redundant input set admits a \((1-\e)\)-variance kernel
approximation using an input set \(r\) times smaller.

\begin{definition}[Parameter redundancy]\label{def:param-redundancy}
Fix a class \(c\).
Given \(\cD=(X,Y)\), the parameters at \(\theta\) are \textbf{\((r,\e)\)-parameter-redundant} if there
exists a subspace \(V \subset \R^P\) with \(\dim(V)=P/r\) such that
\[
\frac{\|\Pi_V(\Phi_X^c)\Pi_V(\Phi_X^c)^\top-K^c_{XX}\|_F}{\|K^c_{XX}\|_F}\le \e,
\]
where \(\Pi_V(\Phi_X^c)\) projects each row of \(\Phi_X^c\) onto \(V\) by right-multiplication: \(\Pi_V (\Phi_X^c) = \Phi_X^c \Pi_V\).
\end{definition}
In other words, for this dataset and class, a \((1-\e)\)-fraction of predictive variation can be captured
using a parameter subspace that is \(r\) times smaller.

\Cref{def:data-redundancy,def:param-redundancy} characterize when an NTK can be approximated in low dimension due to redundancy in data or concentration in a lower-dimensional parameter space. 
To \emph{construct} such a structure in practice, we study how redundancy appears across data subsets and output Jacobians of the NTK.

\subsection{Dataset distillation as gradient subspace selection}\label{sec:dd_as_pca}
We now show that (in a standard one-step / lazy regime) dataset distillation (DD) can be viewed as \emph{selecting a low-dimensional tangent subspace in parameter space} spanned by logit gradients \(\nabla_\theta f(\tilde \vx;\theta)\) at the distilled inputs. 
In this view, distilled inputs \(\tilde X\) act as \emph{inducing points} for the NTK: they determine a parameter-space projector, and the loss incurs an update that is a task-dependent linear combination of those tangent features.

\subsubsection{Bilevel distillation induces a tangent-feature subspace.}
DD aims to synthesize a compact dataset \(\tD=(\tX,\tY)\) of size \(m \ll n\) such that training on \(\tD\) matches training on \(\cD=(X,Y)\). A common bilevel formalization is
\begin{equation}\label{eqn:bilevel}
\tD^* \in \argmin_{\tD}\;
\cL_d\!\left[f\!\left(X;\,\argmin_\theta \cL_p[f(\tX;\theta),\tY]\right),Y\right],
\end{equation}
where the constraints of the optimization problem are determined by the soft or hard biases of the network, and \(\cL_p\) and \(\cL_d\) are designed to optimize the parameters and datasets, respectively.

We work in a frozen-feature (``lazy'') regime around a fixed reference parameter vector \(\theta \in \R^P\), and interpret DD geometrically through the span of logit gradients at \(\theta\).

For an input collection \(X = \{\vx_i\}_{i=1}^m\), we define for each logit gradient \(\phi^c\) the gradient matrices \(\Phi_X^c\in\R^{m\times P}\) by \([\Phi_X^c]_i=\phi^c(\vx_i)\). We also define the \emph{stacked} logit-gradient matrix
\[
\Phi_X := \begin{bmatrix}\Phi_X^1\\ \vdots\\ \Phi_X^C\end{bmatrix}\in\R^{m_{\mathrm{tot}}\times P},
\qquad m_{\mathrm{tot}}:=mC.
\]
For the distilled inputs \(\tX\) we write \(\tPhi := \Phi_{\tX}\) and define the associated tangent subspace
\begin{equation}\label{eqn:VtD_def}
V(\tD)\equiv V(\tX) := \mathrm{colspan}(\tPhi^\top)\subset\R^P,
\end{equation}
where \(\Pi_{\tD}\) is the orthogonal projector onto \(V(\tD)\).
Importantly, \(\tX\) determines the subspace \(V(\tD)\), while \(\tY\) determines how gradients combine within it.

\paragraph{Chain-rule identity.} Here, we show that loss-gradients live in the span of logit-gradients.
Let the distilled inner objective be \(\cL(\theta)=\sum_{i=1}^m \ell[f(\vtx_i;\theta), \vty_i]\) for some per-example loss \(\ell[\cdot,\cdot]\). 
Define the \emph{logit sensitivities}
\[\delta_i(\theta):=\nabla_{f(\vtx_i)}\ell[f(\vtx_i;\theta),\vty_i] \in \R^C,\] 
and stack them into
\(\tdelta(\theta) \in \R^{m_{\mathrm{tot}}}\). By the chain rule,
\begin{equation}\label{eqn:chain_rule_gtilde}
g_{\tD}(\theta):=\nabla_\theta \cL(\theta)
\;=\; \tPhi^\top\,\tdelta(\theta)
\;\in\; V(\tD).
\end{equation}
Thus, although the DD objective is written in terms of loss gradients, those loss gradients are always linear combinations of the \emph{logit} gradients \(\nabla_\theta f^c(\vtx_i;\theta)\) that define our kernel features.

\subsubsection{One-step view: outer progress is controlled by a projection residual.}
The training method we describe in \Cref{sec:method} fixes network parameters at a reference \(\theta\) at
the conclusion of training, approximating a converged state. Using the distilled set, we define a
\emph{frozen} tangent-feature subspace \(V(\tD)=\mathrm{colspan}(\tPhi^\top)\subset\R^P\) (via \(\tX\)) and a
coefficient vector \(\tdelta(\theta)\) (via \(\tY\)), producing the inner update
\(g_{\tD}(\theta)=\tPhi^\top\tdelta(\theta)\in V(\tD)\) as in \eqref{eqn:chain_rule_gtilde}. Thus we do not
model a full training trajectory; instead, we hold the evaluation point \(\theta\) fixed and ask whether the
bilevel objective succeeds at selecting \(\tD\) whose induced one-step update improves typical outer
objectives \(t \sim \mathcal T\) at this same \(\theta\).

This analysis is inherently local about the converged point, so we compare the \emph{realized} one-step
update \(\theta^+(\tD) = \theta - \eta g_{\tD}(\theta)\) to the \emph{best} update available within the same
subspace \(V(\tD)\) under the standard smoothness (quadratic upper-model) approximation. We formulate this
as \emph{regret}: the price of using the update produced by \((\tX,\tY)\) relative to the best
subspace-restricted step, whose achievable decrease is controlled by the projection residual
\(\|(I - \Pi_{\tD}) g_t \|^2\).

\begin{theorem}[One-step smoothness regret bound]\label{thm:dd_proj_final}
Assume \(t \sim \mathcal T\), \(g_t:=\nabla_\theta \mathcal L_t(\theta)\), each \(\mathcal L_t\) is \(L\)-smooth,
and the realized update is \(\theta^+(\tD) = \theta-\eta g_{\tD}(\theta)\) with \(g_{\tD}(\theta)\in V(\tD)\).
Fix \(\tD\) and take expectation over \(t \sim \mathcal T\). Define the one-step smoothness upper model
\[
M_t(\Delta\theta)\;:=\;\langle g_t,\Delta\theta\rangle+\frac{L}{2}\|\Delta\theta\|^2,
\]
and let \(\Delta\theta_t^\star:=\argmin_{\Delta\theta\in V(\tD)} M_t(\Delta\theta)\) denote the best
subspace-restricted step in this model. Then the realized update
\(\Delta\theta_{\tD}:=-\eta\,g_{\tD}(\theta)\in V(\tD)\) satisfies
\begin{equation}\label{eq:headline_regret_bound}
\begin{aligned}
\E_t\!\big[\mathcal L_t(\theta+\Delta\theta_{\tD})-\mathcal L_t(\theta+\Delta\theta_t^\star)\big]\\
\le
\eta\,\E_t\!\Big[
\Big\langle g_t,\frac{\Delta\theta_t^\star}{\eta}\Big\rangle
-\langle g_t,g_{\tD}(\theta)\rangle
\Big] \\
\quad
+\frac{L\eta^2}{2}\Big(
\|g_{\tD}(\theta)\|^2
-\Big\|\frac{\Delta\theta_t^\star}{\eta}\Big\|^2
\Big),
\end{aligned}
\end{equation}
where the first term isolates the penalty for failing to realize the best coefficients within \(V(\tD)\)
(via \(\tY\)), and the second term is the corresponding quadratic-model penalty.

Moreover, the minimizer of \(M_t\) over \(V(\tD)\) is
\[
\Delta\theta_t^\star=-\frac{1}{L}\Pi_{\tD}g_t,
\]
and \(L\)-smoothness implies the best attainable guaranteed decrease within \(V(\tD)\) is
\begin{equation}\label{eq:proj_residual_controls_progress_final}
\begin{aligned}
\mathcal L_t(\theta)-\mathcal L_t(\theta+\Delta\theta_t^\star)
&\ge \frac{1}{2L}\|\Pi_{\tD}g_t\|^2 \\
&= \frac{1}{2L}\Big(\|g_t\|^2-\|(I-\Pi_{\tD})g_t\|^2\Big).
\end{aligned}
\end{equation}
In particular, when the coefficients produced by \(\tY\) make \(g_{\tD}(\theta)\) close to \(\Pi_{\tD}g_t\)
(e.g.\ under soft-label realizability), the one-step regret is small and maximizing expected one-step
progress reduces to minimizing the expected projection residual \(\E_t[\|(I-\Pi_{\tD})g_t\|^2]\).
\end{theorem}

\begin{proof}
See \Cref{appsec:thm-proof}.
\end{proof}

\begin{remark}[Coefficient realizability.]
Equation~\eqref{eqn:chain_rule_gtilde} shows \(g_{\tD}(\theta)\) is always in \(V(\tD)\), but it need not equal the \emph{best} restricted update \(\Pi_{\tD} g_t\). 
Intuitively, \(\tX\) chooses the subspace (i.e.\ which tangent features are available), and \(\tY\) chooses coefficients within that subspace through through \(\tdelta(\theta)\). 
When \(\tY\) is sufficiently expressive (e.g.\ soft labels / locally linearized losses), DD can closely approximate \(\Pi_{\tD}g_t\).
\end{remark}

\subsubsection{Competing objectives imply a PCA subspace of gradient covariance.}
\begin{corollary}[Competing objectives \(\Rightarrow\) PCA subspace of gradient covariance]\label{cor:pca_final}
Let \(G:=\E_t[g_t g_t^\top]\) with eigenvalues \(\lambda_1\ge\cdots\ge\lambda_P\).
Among all \(r\)-dimensional subspaces \(V\) (a relaxation of realizable \(V(\tD)\); note
\(\dim V(\tD)\le\mathrm{rank}(\tPhi)\le m_{\mathrm{tot}}\)), the minimizer of
\[
\E_t[\|(I-\Pi_V)g_t\|^2]=\mathrm{tr}(G)-\mathrm{tr}(\Pi_V G)
\]
is the top-\(r\) eigenspace of \(G\). Moreover, if \(\mathrm{tr}(\Pi_{V^\star}G)-\mathrm{tr}(\Pi_VG)\le\delta\), then
\[
\E_t[\|(I-\Pi_V)g_t\|^2]\le \sum_{j>r}\lambda_j+\delta.
\]
\end{corollary}
\begin{proof}
See \Cref{appsec:cor-proof}.
\end{proof}

\subsubsection{Inducing-point view and kernel fidelity.}
The tangent subspace viewpoint also explains why distilled inputs behave like \emph{inducing points} in kernel space. 
Fix a scalar output (e.g.\ a logit \(c\)) and let \(\Phi:=\Phi_X^c\in\R^{n\times P}\) and \(\tPhi^c:=\Phi_{\tX}^c\in\R^{m\times P}\) denote the corresponding gradient feature matrices at \(\theta\).

Given any parameter-space projector \(\Pi\) (e.g.\ \(\Pi=\Pi_{\tD}\)), define the projected-feature class kernel
\begin{equation}\label{eqn:projected_kernel_def}
K^{c,\Pi}_{XX} := (\Phi\Pi)(\Phi\Pi)^\top = \Phi\Pi\Phi^\top.
\end{equation}
If we choose \(\Pi\) as the projector onto the span of distilled class features,
\begin{equation}\label{eqn:Pi_from_tPhi}
\Pi_{\tD}^c := (\tPhi^c)^\top \big(\tPhi^c(\tPhi^c)^\top\big)^\dagger \tPhi^c,
\end{equation}
then \eqref{eqn:projected_kernel_def} becomes the Nystr\"om / inducing-point form
\begin{equation}\label{eqn:nystrom_identity}
\begin{split}
K^{c,\Pi_{\tD}^c}_{XX}
= K^c_{X\tX}\,\big(K^c_{\tX\tX}\big)^\dagger\,K^c_{\tX X}, \\
K^c_{X\tX}:=\Phi(\tPhi^c)^\top, \qquad
\;K^c_{\tX\tX}:=\tPhi^c(\tPhi^c)^\top.
\end{split}
\end{equation}
Thus, selecting \(\tX\) selects an inducing set in the tangent-feature kernel.

Finally, kernel fidelity is controlled by how well \(\Pi\) preserves gradient features:
\begin{equation}\label{eqn:kernel_error_bound}
\|K^c_{XX}-K^{c,\Pi}_{XX}\|_F
=
\|\Phi(I-\Pi)\Phi^\top\|_F
\;\le\;
\|\Phi\|_F\,\|\Phi(I-\Pi)\|_F.
\end{equation}
In particular, misalignment of \(\Pi\) with the dominant right-singular subspace of \(\Phi\) directly translates into kernel approximation error.

\begin{proposition}[Energy-gap decomposition in gradient-feature space]\label{prop:constrained_pca_decomp}
Let \(\Phi=U\Sigma W^\top\) and let \(\Pi^\star:=W_r W_r^\top\) be the rank-\(r\) PCA projector (top \(r\) right singular vectors). 
Then for any rank-\(r\) projector \(\Pi\) (including realizable choices such as
\(\Pi=\Pi_{\tD}^c\)),
\begin{equation}\label{eqn:gap_decomp}
\|\Phi(I-\Pi)\|_F^2
= \underbrace{\sum_{j>r}\sigma_j(\Phi)^2}_{\text{PCA tail}}
+ \underbrace{\Big(\mathrm{tr}(\Phi^\top\Phi\,\Pi^\star)-\mathrm{tr}(\Phi^\top\Phi\,\Pi)\Big)}_{\text{captured-energy gap (misalignment)}}.
\end{equation}
The second term is nonnegative and vanishes if and only if \(\Pi\) captures as much feature energy as the PCA subspace, providing a direct quantitative measure of subspace misalignment relevant for kernel fidelity via \eqref{eqn:kernel_error_bound}.
\end{proposition}
\begin{proof}
See \Cref{appsec:prop-proof}.
\end{proof}

\noindent\textbf{Interpretation.}
\Cref{thm:dd_proj_final} shows that one-step outer progress is governed by how much of \(g_t\) lies in \(V(\tD)\); \Cref{cor:pca_final} identifies the relaxed optimal subspace as a PCA subspace of gradient covariance; and \Cref{prop:constrained_pca_decomp} links realizable subspaces induced by distilled inputs to explicit feature- and kernel-space approximation error. Full proofs and the link from task gradients \(g_t\) to \(\Phi\) under linearized/squared-loss models appear in \Cref{appsec:pca_proof}.

\subsection{Spectral structure of the kernel}\label{sec:spectral}
For class \(c\), let \(K^c = U^c \Sigma^c (U^c)^\top \in \R^{n \times n}\) be the class kernel with truncation rank \(r_g\) for some small \(\e\) (as in \Cref{def:data-redundancy}).
Using kernel clustering (spectral clustering with the adjacency matrix given by \(K^c\)), partition the \(n\) samples into \(H\) clusters \(\{h_1, \ldots, h_H\}\) with index sets \(\cI_1, \ldots, \cI_H\). For cluster \(h_i\), define the \textbf{local class kernel} as the restriction of the (global) class kernel to local indices: \(K^c_i = K^c|_{\mathcal{I}_i} = U^c_i \Sigma^c_i (U^c_i)^\top\) with truncation rank \(r_i\).

Let \(A^{{(r)}}\) denote the first \(r\) columns of the matrix \(A\). For now, we introduce the following properties and assume that class-structured data follows them:
\begin{enumerate}
    \item[\textbf{(A)}] Local variance is almost entirely contained within the global eigenspace. For each cluster \(h_i\), 
    \begin{equation}
        \frac{ \sum_{j=1}^{r_i} [\Sigma_i]_{jj} \cdot \| \Pi_i^{\text{glob}}( \vu_i^{j}) \|^2}{ \sum_{j=1}^{r_i} [\Sigma_i]_{jj} } \approx 1, \quad \text{where}
    \end{equation}
    \begin{itemize}
        \item \([\Sigma_i]_{jj}\) is the \(j^{\text{th}}\) eigenvalue local to cluster \(h_i\),
        \item \(\vu_i^{j} \in \R^{|\cI_i|}\) is the \(j^\text{th}\) eigenvector local to cluster \(h_i\), and
        \item \(\Pi_i^{\text{glob}}\) is the orthogonal projection onto \(\text{span}(U^{(r_g)}|_{\cI_i})\) (after orthonormalization).
    \end{itemize} 
    This property implies that \(\text{span}(U_i) \subseteq \text{span}(U |_{\cI_i})\) approximately for a given cluster \(h_i\).
    \item[\textbf{(B)}] Local eigenspaces collectively do \textbf{not} span global eigenspace, whereby there exist some global principal directions that are poorly represented by the union of the spans of global clusters. 
    Formally, let \(\hat{U}_i \in \R^{n \times r_i}\) be the zero-padded lifting of \(U_i^{(r_i)}\), and let 
    \(\Pi^{\text{loc}}\) denote the orthogonal projector onto \(\text{span}([\hat{U_1}, \ldots, \hat{U}_H]).\)
    Then there exist \(j \leq r_g\) such that, for \(\delta > 0\),
    \[\| \Pi^{\text{loc}} (\vu^j)\|^2 \leq 1 - \delta.\]
    Equivalently, \(\text{span}(U^{(r_g)}) \not \subseteq \text{span}([\hat{U_1}, \ldots, \hat{U}_H]) \), where \(\hat{U}_i\) are the lifted (zero-padded) local eigenvectors.
\end{enumerate}
Qualitatively, properties \textbf{(A)} and \textbf{(B)} describe a hierarchical redundancy structure: within-cluster variance is concentrated (enabling local compression via data redundancy, \Cref{def:data-redundancy}), but cross-cluster relationships span a complementary subspace (requiring global eigenmodes to maintain kernel fidelity). These assumptions are supported by empirical findings in \Cref{fig:local_global_composition} and form the theoretical basis for \Cref{alg:local_global_comp}, which explicitly constructs gradient representatives preserving both spectral regimes.

\section{Method}\label{sec:method}

The DNTK method addresses the complexity profile in \Cref{sec:krr-complexity} through three successive reductions. Starting from the original training set \(X\) of size \(n\), we first apply dataset distillation in input space to obtain a distilled set \((\tX, \tY)\) of size \(m \ll n\). For each class \(c\), we then form a projected gradient matrix \(\tPhi_X^c \in \R^{m \times k}\) by applying a distance-preserving random projection \(g : \R^P \to \R^k\) to the per-sample parameter gradients, thereby reducing the effective parameter dimension from \(P\) to \(k\).
Finally, we perform a second round of distillation in gradient space, replacing \((\tPhi, \tY)\) with a smaller synthetic set \((\hat{\Phi}, \hat{Y})\) of size \(s \ll m\), which is used in the final kernel ridge regression solve.

\subsection{Data distillation}\label{sec:method_dd}

We instantiate the distillation framework of \Cref{sec:dd_as_pca} using \textbf{Wasserstein Metric Dataset Distillation (WMDD)} \cite{liu_wmdd_2025}.
WMDD solves a surrogate objective based on feature matching: it synthesizes data whose intermediate representations are distributionally close to the original data in both input and feature space.
Concretely, WMDD minimizes
\[
\mathcal{L}(\tilde{X}) = \mathcal{L}_{\text{feature}} + \lambda_{\text{BN}} \mathcal{L}_{\text{BN}},
\]
where \(\cL_{\text{feature}}\) matches synthetic features to Wasserstein barycenters of real features, and \(\cL_{\text{BN}}\) aligns batch normalization statistics across layers. The WMDD process generates a soft label for each distilled data point for a set \(\tY\) of distilled labels. 
This objective is aligned with our subspace view in \Cref{sec:dd_as_pca}: by matching intermediate representations (and BN statistics), WMDD tends to produce distilled points whose gradient span is better aligned with the dominant directions of the full gradient matrix, reducing the misalignment term in \Cref{thm:dd_proj_final}. 
Implementation details appear in \Cref{appsec:wmdd}.

\subsection{Random projection}

The Johnson-Lindenstrauss (JL) lemma famously states that, for a desired error bound \(\e_\JL\) and an integer \(k > (8 \ln n) / \e_\JL^2\), there exists a linear map \(g: \R^P \to \R^k\) such that 
\[(1-\e_\JL)\|\vu-\vv\|^2 \leq \|g(\vu) - g(\vv)\|^2 \leq (1+\e_\JL)\|\vu-\vv\|^2\] 
for any \(\vu, \vv \in \R^P\) \cite{dasgupta-jl-2003}.

This is commonly achieved using a random orthonormal projection from \(\R^P\) to \(\R^k\), where \(k \ll P\). Because \(\langle u, v \rangle = \frac{1}{2}(\|u\|^2 + \|v\|^2 - \|u-v\|^2)\), JL distance preservation over a finite set also yields approximate preservation of the inner products that define the NTK entries. 

For this projection, we generate a random orthonormal matrix \(Q \in \R^{P \times k}\), then set \(g(u)=\sqrt{P/k}\;Q^\top u\), where the \(\sqrt{P/k}\) factor offsets the scaling by \(Q\).
This yields our approximated features \(\tPhi_X^c = g(\Phi_X^c)\) for our dataset \(X\) and corresponding approximate kernel \(\tK^c = \tPhi_X^c (\tPhi_X^c)^\top \in \R^{m \times m}\).

\begin{remark}
In practice, JL gives a high-probability bound on \(\|\tK^c - K^c\|\) (and thus on spectral quantities like effective rank), so the redundancy estimates computed from \(\tK^c\) track those of \(K^c\) up to an error defined by \(\e_\JL\) and the kernel's eigenspectrum. 
See \Cref{sec:jl-redundancy} for details.
\end{remark}

\subsection{Gradient distillation}\label{sec:method_gd}
Motivated by the empirical local-global kernel structure analyzed in \Cref{sec:local-global}, we introduce \textbf{local-global gradient distillation} (\Cref{alg:local_global_comp}), which outputs synthetic projected gradients \(\hat{Phi}\) and synthetic targets \(\hat{Y}\) as linear combinations of the original projected gradients. This distillation allows a set of gradients to have some representatives covering the kernel clusters, which compose a majority of the global variance, and some covering the connective gaps shown in  \Cref{fig:local_global_composition}. 

\begin{remark}\label{rem:syn-grads}
Let \(\vu\) be unit norm. If \(\vu\) is an eigenvector of \(K = \frac{1}{k} \Phi \Phi^\top\) with an eigenvalue \(\lambda\), then 
\[\hat{\phi} = \Phi^\top \vu = \sum_{i=1}^n u_i \phi_i\]
satisfies \(\|\hat{\phi}\|^2 = k \lambda\). That is, \(\hat{\phi}\) generates the principal direction \(\vu\) in kernel space. We use this to compute our distilled gradients directly.
\end{remark}

\Cref{alg:local_global_comp} proceeds in six stages:
\begin{enumerate}
    \item[\textbf{(1)}] Compute the average class kernel \(\bar{K} = \frac{1}{C} \sum_c K^c\), and partition \(H\) cluster indices \(\{\cI_h\}_{h=1}^H\) by spectral clustering.
    
    \item[\textbf{(2)}] Record the global eigendecomposition \(U \Sigma U^\top\) of \((\bar{K})\), and set \(r_g := r_{\text{trunc}} (\bar{K}, 1 - \tau_v)\), where \(\tau_v\) is a user-defined threshold.

    \item[\textbf{(3)}] Compute each global rank for which that rank's eigenvectors are not aligned with \textit{any} local cluster above the threshold \(\tau_g\),\footnote{\(\tau_g\) weighs between intra- and inter-cluster focus: the larger the \(\tau_g\), the more gradients are classified as gaps.} another user-defined threshold. ``Alignment'' is quantified by the fraction of restricted energy captured by the local top-eigenspace: \(c_j := \max_h \| P_h (u_j |_{\cI_h})\|^2 / \|u_j|_{\cI_h}\|^2\), and define gap directions \(\cG = \{j: c_j < \tau_g\}\).

    \item[\textbf{(4)}] Distill representatives of the \textbf{clusters} by setting \(\hat{\phi} := \Phi[\cI_h]^\top \vu \in \R^{k \times C}\), i.e., a linear combination of the projected gradients in that cluster (see \Cref{rem:syn-grads}). Similarly, set \(\hat{Y} := Y[\cI_h]^\top \vu \in \R^C\), a soft target.

    \item[\textbf{(5)}] Distill representatives of the \textbf{gaps}  by setting \(\hat{\phi}\) again according to Remark \ref{rem:syn-grads} for the global gradients \(\Phi\) and the eigenvectors indexed by \(\cG\).

    \item[\textbf{(6)}] Orthogonalize the accumulated set of eigenvectors and keep only the \(s\) gradients with non-redundant indices. 
\end{enumerate}
This process results in \(s \ll m\) distilled gradients that collectively span the kernel's subspace better than existing optimal sampling methods. The full algorithm details can be found in \Cref{appsec:local_global_comp}.

\subsection{Kernel solving}
These three methods, run sequentially, yield tensors \(\hat{\Phi}^c\) of ``thrice-distilled'' gradients for each class \(c\).
We then use \((\hat{\Phi}, \hat{Y})\) in the place of \((\Phi, Y)\) in KRR; see \Cref{appsec:ridge-regression}.
\section{Experiments}\label{sec:experiments}

 We evaluate DNTK approximations, demonstrating that kernel models preserve \textbf{high predictive fidelity} while exhibiting \textbf{substantial data and parameter redundancy}. These findings allow for further data- and parameter-reduction techniques that minimally affect performance.

\subsection{Accuracy, fidelity, and error}

The distilled dataset \(\tX\) is distilled from a fixed dataset (ImageNette) and computed on a fixed model architecture (ResNet-18). We then evaluate kernel representations on this same distilled set using two models: one pretrained on real data, and one trained solely on the distilled data. In both cases, the kernel is computed from gradients evaluated at the distilled inputs.

In \Cref{fig:size-acc-fid-mse}, accuracy and fidelity both quickly saturate to the level of the original model with relatively few training points. The rapid saturation in both regimes suggests the kernel matrices \(\tK^c_{\tX \tX}\) exhibit low-rank structure. We quantify this through spectral analysis of the condition number and minimum eigenvalue.

\begin{figure}[!h]
    \centering
    \includegraphics[width=0.9\linewidth]{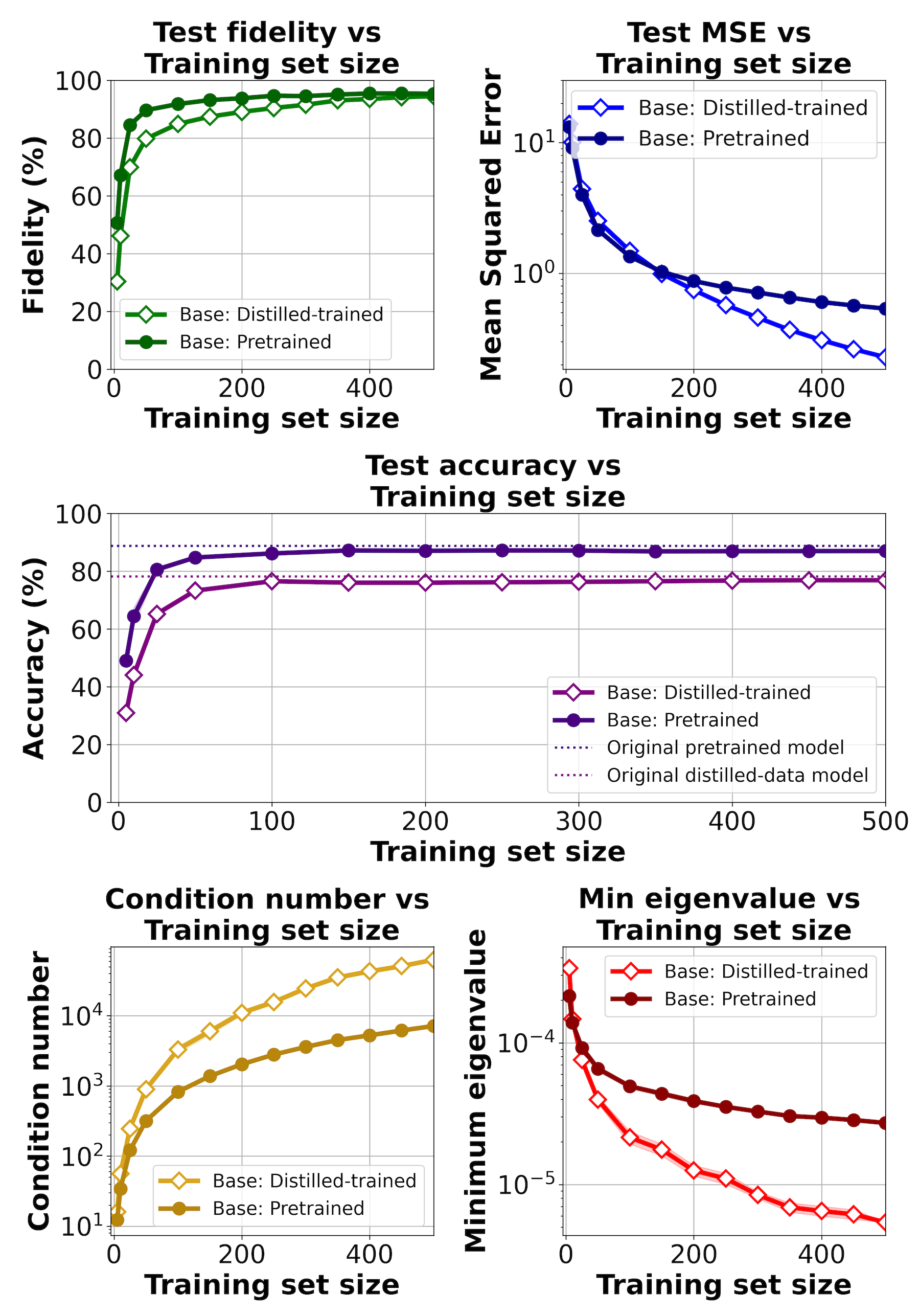}
    \caption{Kernel-model accuracy metrics as a function of sample size, (where samples are are taken evenly across classes from the 500 available distilled gradients). Experiments are run on the ImageNette dataset and ResNet-18 model.
\textbf{Test fidelity:} fraction of matched predictions between \(f_K\) and \(f\).
\textbf{Test MSE:} computed from predicted logit differences. 
\textbf{Test accuracy:} correct predictions on an unseen test set. 
\textbf{Condition number and minimum eigenvalue:} stability of kernel 
matrices \(\tilde{K}^c_{\tilde{X} \tilde{X}}\) averaged across classes.
Across all metrics, we find that a pretrained base model results in lower loss and better-conditioned kernel than a distilled-data base model, although the performance differs by 10\% if only the distilled-data model is available.}
\label{fig:size-acc-fid-mse}
\end{figure}

\subsection{Data redundancy}

The truncation ranks of class kernels are significantly lower than even that exploited by initial data distillation, as shown in \Cref{fig:svd}.  
The exponential decay of singular values implies that the truncated SVD approximation \(\tK^c_{\tX \tX}\) retains most kernel variance, suggesting class kernel can be accurately represented in a low-dimensional subspace.

Equivalently, in kernel ridge regression, the solution \(\valpha\)
minimizing \(\|K_{XX} \; \valpha - Y\|^2 + \lambda_\reg \valpha^\top K_{XX} \valpha\)
is dominated by the leading eigenmodes of \(K_{XX}\).
Specifically, when \(\lambda_i \ll \lambda_\reg = 10^{-4}\) (our regularization parameter), those directions are effectively suppressed.

\begin{figure}[!h]
    \centering
    \includegraphics[width=0.9\linewidth]{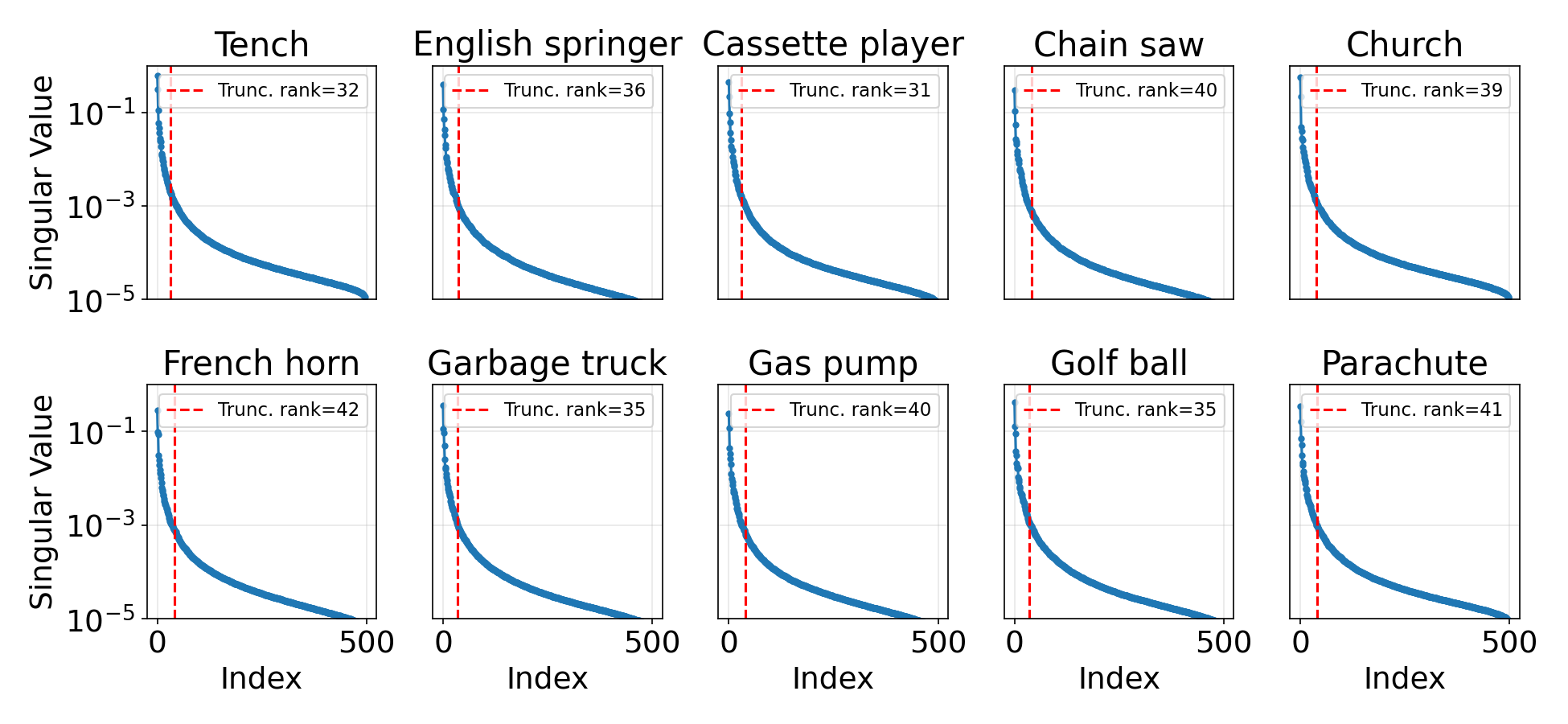}
    \caption{Singular values of class kernels reduce exponentially, with truncation ranks between 31 and 41, denoting \((12, 0.05)\)- to \((16, 0.05)\)-data redundancy, depending on the class.}
    \label{fig:svd}
\end{figure}

Motivated by this observation, we compute rank-\(r\) SVD approximations of the class kernels in \Cref{fig:rank-acc-fid-mse} (\Cref{appsec:other-experiments}) and evaluate the testing accuracy as a function of \(r\).

\subsection{Local-global composition}\label{sec:local-global}
Here we implement \Cref{alg:local_global_comp} as discussed in \Cref{sec:method_gd}
Figure~\ref{fig:size-acc-fid-mse} shows that \Cref{alg:local_global_comp} substantially outperforms gradient sampling baselines (leverage, k-means, random, FPS) across compression ratios. At 100× compression (five distilled gradients), it achieves 76\% accuracy and 78\% fidelity, while baselines plateau well below this ceiling. Moreover, the algorithm matches the full \(f_K\) accuracy with far fewer gradients than any competing method.

\begin{figure}[!ht]
    \centering
    \includegraphics[width=\linewidth]{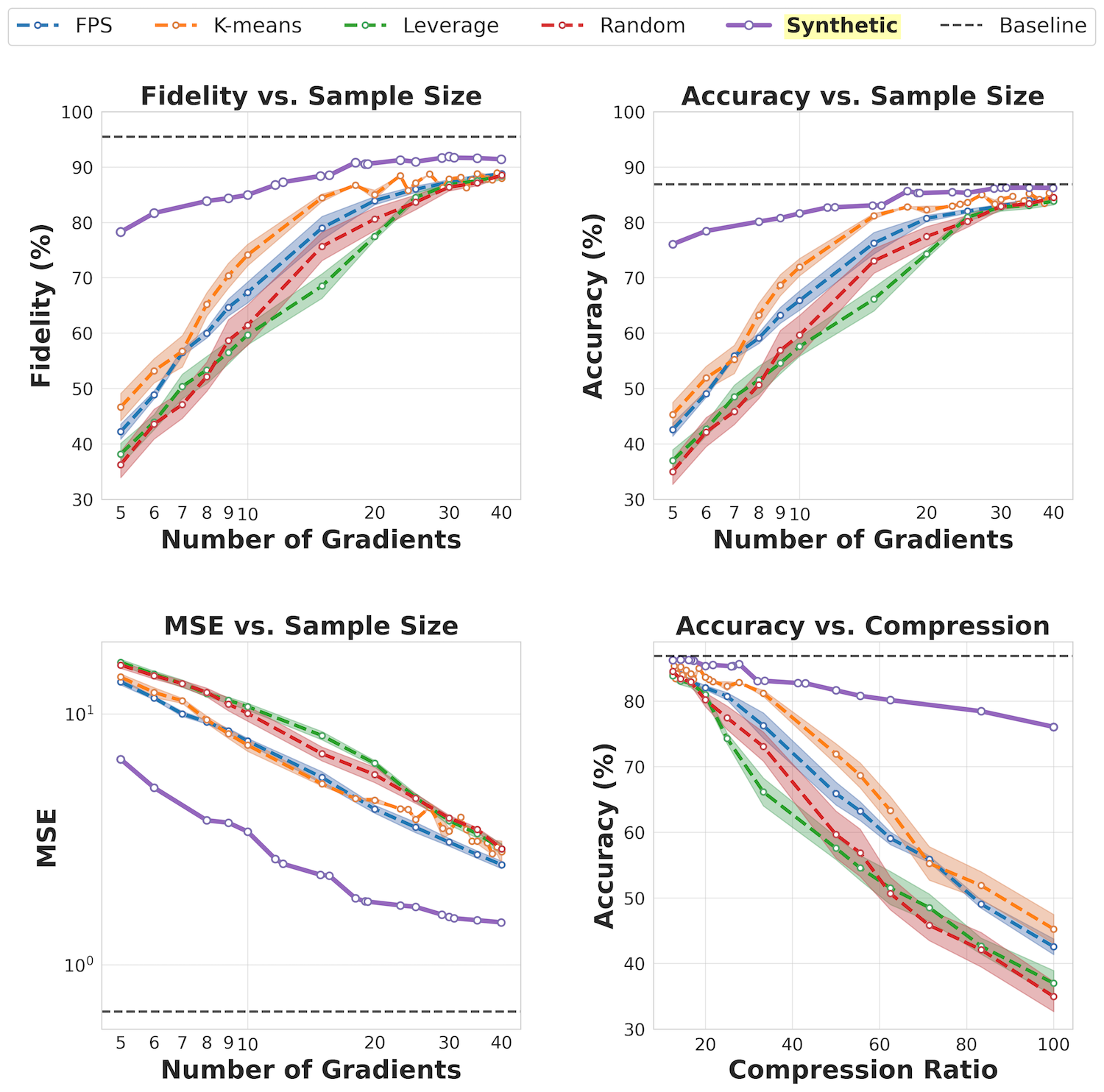}
    \caption{Test metrics (fidelity, accuracy, and MSE) taken from \Cref{fig:size-acc-fid-mse}. Compression ratio (bottom right) is defined as \(m/s\), where \(m\) is the number of original gradients and \(s\) is the number of gradients distilled by \Cref{alg:local_global_comp}.}
    \label{fig:comparison-plot}
\end{figure}

This performance gain stems from the algorithm's ability to capture kernel structure more completely. Figure~\ref{fig:coverage-plot} (\Cref{appsec:other-experiments}) shows that distilled gradients achieve higher subspace variance coverage and lower reconstruction error than baselines, particularly at high compression. By construction, the algorithm synthesizes gradients spanning both intra-cluster concentrated modes (step \textbf{(4)}) and inter-cluster gap modes (step \textbf{(5)}), preserving the kernel's full spectral range.

These results confirm the spectral structure hypothesized in \Cref{sec:spectral}. Figure~\ref{fig:local_global_composition} illustrates properties \textbf{(A)} and \textbf{(B)} empirically: at \(\e=5\%\) truncation, local cluster variance projects almost entirely onto the global eigenspace (property \textbf{(A)}, top panel), yet roughly \(14\%\) of global variance lies outside the union of all local clusters (property \textbf{(B)}, bottom panel). This local-global gap exposes a significant performance gap unfilled by methods relying solely on clustering or leverage scores, which capture only local structure.

\begin{figure}[!h]
    \centering
    \includegraphics[width=\linewidth]{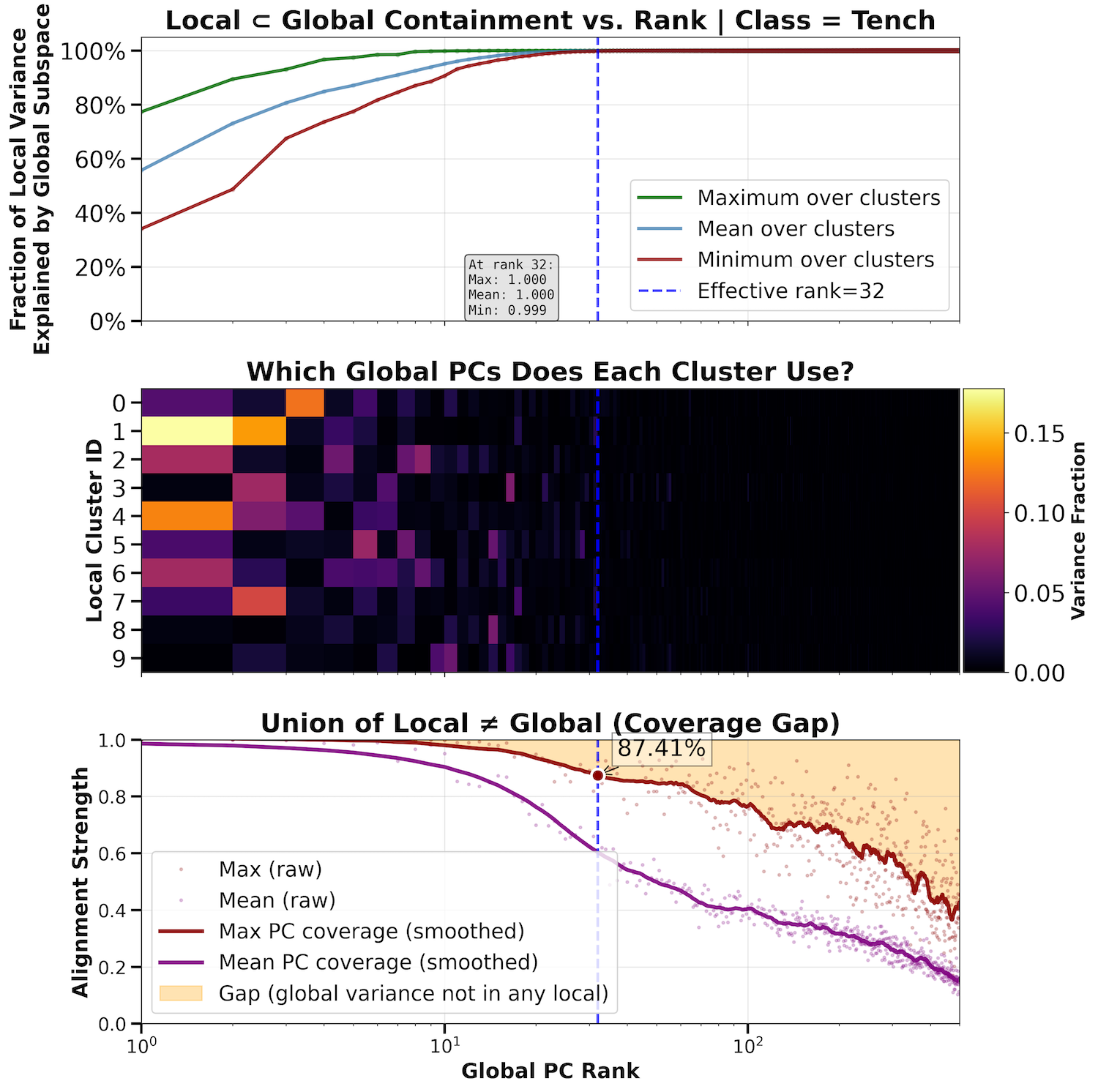}
    \caption{Relationships between the spans of local and global eigenvectors across 10 clusters on the ``tench'' class, whose global truncation rank (at 95\% explained variance) is 32. 
    \textbf{Top}: Local eigenvectors \(\{\vu_i^{j}\}_{j=1}^{r_i}\) project almost entirely onto the subspace spanned by global eigenvectors \(\{\vu^{(r)}\}_{r=1}^{r_g}\) as rank \(r_g\) increases, demonstrating property \textbf{(A)} of \Cref{sec:spectral}. Curves show the fraction of variance-weighted local eigenvectors contained in the first \(r\) global principal components, with maximum, mean, and minimum over clusters approaching 100\% near the truncation rank.
    \textbf{Middle}: Variance decomposition showing which global PCs each cluster uses. Cell \((i,j)\) displays the variance of cluster \(i\)'s kernel along global PC \(j\), computed as \((\vu^{j}|_{\mathcal{I}_i})^\top K_i (\vu^{j}|_{\mathcal{I}_i})\) normalized by \(\text{tr}(K_i)\). Bright regions indicate the global dimensions that explain each cluster's structure.
    \textbf{Bottom}: Coverage gap demonstrating property \textbf{(B)}. For each global PC, the curves show the maximum (dark red) and mean (purple) alignment strength \(\| P_i(\vu^{j}|_{\mathcal{I}_i}) \|^2\) across all clusters. The orange shaded region represents global variance directions that are poorly covered by any local eigenspace, revealing that roughly \(\e=\)12-15\% of global structure is not captured by the union of local clusters at the truncation rank. Analogous patterns across all ten classes appear in \Cref{appsec:contain-gaps}.
    }
    \label{fig:local_global_composition}
\end{figure}

Grid search over \((\tau_v, \tau_g, H)\) (Appendix~\ref{appsec:grid-search}) reveals that Pareto-optimal configurations are dataset-dependent, with performance sensitive to local cluster structure. This indicates a convenient lever that practitioners can use to adapt our method to varying degrees of local/global decomposition to suit particular datasets. 
\section{Conclusion}\label{sec:conclusion}
We demonstrate that neural tangent kernel (NTK) representations for modern neural networks can be computed at practical scale, on datasets with \(O(10^3)\) to \(O(10^4)\) samples, by exploiting redundancy in both the data and parameter dimensions. By combining data-space distillation, Johnson–Lindenstrauss random projection, and gradient-space distillation into a unified compression pipeline, we reduce the computational and storage cost of NTK construction by 4–6 orders of magnitude compared to naïve full NTK computation on image classification benchmarks, and by a further 1–2 orders of magnitude beyond parameter-space random projection alone, while preserving kernel fidelity. Our theoretical analysis justifies each stage of the approximation, and a spectral study shows that per-class NTK matrices have effective ranks far smaller than the dataset size, indicating that gradient features concentrate in low-dimensional subspaces of parameter tangent space. Leveraging this structure, we introduce a gradient synthesis algorithm that achieves up to 100\(\times\) additional compression relative to distilled gradients, attaining 76\% accuracy on ImageNette with only five synthetic gradients and consistently outperforming sampling-based baselines. 

Together, these results extend NTK-based analysis beyond small-scale settings and suggest that, for practical networks, kernel representations can be manipulated efficiently using standard linear-algebraic tools.
\section*{Impact Statement}
This work provides results that enable more efficient computation of the NTK, an important tool for understanding and verifying neural networks. To our knowledge, there are no serious risks associated with the release of these results.
\section*{Acknowledgements}
We'd like to acknowledge Nick Hengartner, Robyn Miller, and Juston Moore for their helpful conversations and support on this project.
\clearpage
\pagebreak

\bibliography{bibliography}
\bibliographystyle{icml2026}

\newpage
\appendix
\onecolumn
The appendices are organized as follows:
\begin{itemize}
    \item \Cref{appsec:ntk-prelims}: we introduce the NTK, including the relevant spaces and relationships, in a more rigorous way than was defined in the main paper.
    \item \Cref{appsec:error_bounds}: we quantify the expected error bounds in approximating a function by its kernel representation, including conditions under which we can expect a good approximation.
    \item \Cref{appsec:pca_proof}: we provide a formal proof of \Cref{thm:dd_proj_final}.
    \item \Cref{appsec:other-experiments}: we reproduce the experiments in \Cref{sec:experiments} on a range of other (dataset, model) pairs.
    \item \Cref{appsec:local_global_comp}: we provide the formal statement of the local-global composition algorithm, including complexity and a grid search analysis on the (ImageNette, ResNet-18) experiment.
\end{itemize}

\section{NTK preliminaries}
\label{appsec:ntk-prelims}
\subsection{Defining the network}
Let \(\Theta\) be a \(P\)-dimensional Riemannian manifold of parameters, and let \(\cX\) and \(\cY\) be data and output spaces, respectively. Furthermore, let \(\cR: \Theta \to (\cX \to \cY)\) be a realization function that maps \(\theta \in \Theta\) to its corresponding network \(f(\cdot \;; \theta): \cX \to \cY\), so that \(\cF := \cR(\Theta)\) is the set of functions from \(\cX \to \cY\) realizable by some \(\theta \in \Theta\). 

\subsubsection{Image classification}
In the case of image classification, \(\cX = \R^{\din}\) is the space of 2D images such that \(\din = (\text{length} \cdot \text{height} \cdot \text{channels})\), while \(\cY = \R^{d_\out}\) is the log-probability space over \(d_\out=C\) classes. The network \(f\) accepts and processes an image \(\vx \in \cX\) by creating sequential activations \(A^{(\ell)}\) across layers \(\ell = 0, 1, \ldots, L\), each with width \(d_\ell\), by the operation
\begin{equation}
    A^{(\ell+1)} = \sigma\left(h^{(\ell)}(A^{(\ell)}; \theta^{(\ell)})\right),
\end{equation}
where \(A^{(0)}(\vx)=\vx\), \(\sigma\) is an activation function like \(\tanh\) or ReLU, and $h^{(\ell)}: \mathbb{R}^{d_\ell} \to \mathbb{R}^{d_{\ell+1}}$ is the transformation at layer $\ell$ (dense, convolutional, attention, etc.) with trainable parameters $\theta^{(\ell)} \in \mathbb{R}^{P_\ell}$. The vector of flattened parameters \(\theta\) thus has size \(P = \sum_{\ell=0}^{L-1} P_\ell\).

The functional output \(f(\vx; \theta)\) is given by \(\log(\softmax(\vz))\), where the final activation \(\vz=A^{(L)}\) is a vector in \(\mathbb{R}^C\) given by 
\[A^{(L)} = h^{(L-1)}(  \ldots \sigma(h^{(0)}\vx) )\]
and the softmax function is given element-wise by \(\softmax(z_i) = e^{z_i} / \left(\sum_j e^{z_j} \right)\).

\subsubsection{Tangents and tangent spaces}
Consider a fixed set of parameters \(\theta \in \Theta\) and the function \(f(\cdot; \theta) = \cR(\theta) \in \cF\) it realizes. For notational clarity, we write \(f(\cdot \;; \theta)\) when emphasizing the dependence on which parameters we are approximating, and \(f(\cdot)\) when \(\theta\) is clear from context. 

The tangent spaces \(\cT_\theta \Theta\) and \(\cT_f \cF\) consist of infinitesimal changes to parameters and functions, respectively. These spaces are connected by the differential of the realization map, \(d\mathcal{R}: \mathcal{T}_\theta \Theta \to \mathcal{T}_f \cF\). For a parameter perturbation \(\delta\theta\), the differential produces a function perturbation \(\delta f = d\mathcal{R}(\delta\theta)\) whose value at input \(\vx\) is given by
\begin{equation}\label{eqn:param-tangent}
\delta f = d\mathcal{R}(\delta\theta) = \langle \nabla_\theta f(\vx; \theta), \delta\theta \rangle_\Theta
\end{equation}

Here, \(\nabla_\theta f(\vx; \theta)\) is the gradient of the evaluation functional at \(\vx\), i.e., the map \(\theta \mapsto f(\vx; \theta)\) that outputs the network's prediction at the specific input \(\vx\). This gradient indicates which direction in parameter space most increases the output at that particular point. The differential \(d \cR\) thus measures how parameter changes translate into function changes across all inputs simultaneously.

\subsection{The neural tangent kernel}

Using the same notation as the previous subsection, \(\nabla_\theta f(\vx; \theta)\) is the gradient of the evaluation functional at \(\vx\) with respect to the parameter.
NTK is defined based on the gradient of the network's output with respect to the parameter.

For any two inputs \(\vx, \vx'\), the \textbf{neural tangent kernel at \(\theta\)} is given by 
\begin{equation} K(\vx, \vx') = \langle \nabla_\theta \cR(\theta)(\vx), \nabla_\theta \cR(\theta)(\vx) \rangle_{\Theta} =\langle \nabla_\theta f(\vx ; \theta), \nabla_\theta f(\vx;\theta) \rangle_\Theta, 
\label{eqn:ntk-space-def}
\end{equation}
where \(\langle \cdot, \cdot \rangle\) is the inner product defined on \(\Theta\) measuring the network's response to infinitesimal parameter change at different inputs.

\subsection{Ridge regression}
\label{appsec:ridge-regression}
Since \(K\) is a symmetric, positive-definite kernel in the infinite-width limit \cite{jacot_ntk_2018}, it defines a reproducing kernel Hilbert space (RKHS) \(\cH_K\) by the Moore–Aronszajn theorem \cite{moore-aronszajn}. While a finite-width trained network \(f\) does not lie exactly in \(\cH_K\), we can approximate it via kernel ridge regression: given training points \(X = \{\vx_i\}_{i=1}^n\) with labels \(Y = \{f(\vx_i; \theta)\}_{i=1}^n\) and regularization \(\lambda_{\reg} \geq 0\), the per-class ridge estimator solves
\begin{equation}
    \valpha^c = (K^c + \lambda_\reg I)^{-1} Y^c
    \label{eqn:krr-general}
\end{equation}
for each class \(c\), yielding the predictor \(f_K^c(\cdot) = \sum_{i=1}^n \alpha_i^c K^c(\vx_i, \cdot)\).

For efficient computation, we decompose \(K^c_{XX} = U \Sigma U^\top\) and apply the Woodbury identity:
\begin{equation}\label{eqn:woodbury}
    \valpha^c = U(\Sigma + \lambda_\reg I)^{-1} U^\top Y^c.
\end{equation}
When using a rank-\(r\) approximation (\Cref{sec:experiments}), we substitute \(U^{(r)}\)  for \(U\) and \(\Sigma^{(r)}\) for \(\Sigma\) in \eqref{eqn:woodbury}.

The predictive mean for class \(c\) at test point \(\vx_*\) is
\begin{equation}\label{eqn:predictive_mean}
f_K^c(\vx_*) = \sum_{i=1}^m \alpha_i^c K^c(\vx_i, \vx_*) = \frac{1}{k} \phi(\vx_*)^\top \Phi^c_{\tX} \valpha^c,
\end{equation}
where \(\phi(\vx_*)\) is the gradient at the test point. The full predictor stacks these across classes: \(f = [f_K^1, \ldots, f_K^C]\).

\section{Error bounds}
\label{appsec:error_bounds}

The NTK in \eqref{eqn:ntk-space-def} describes interesting facets of model training dynamics, most notably concluding that models are lazy trainers that follow kernel gradient descent in the infinite-width limit. Our goal here, however, is to justify that a network defined by a \textit{fixed} parameter set can be approximated and analyzed on a given dataset as a linear combination of kernels.

Suppose, for a given parameter set \(\theta\), we wish to use the form given \Cref{appsec:ntk-prelims} to represent the target function \(f = \cR(\theta) \in \cF\) that is generally not in \(\cH_K\) by a finite linear combination from \(\cB\). We seek a finite-dimensional approximation within the span of partial kernel evaluations at training points, obtained via kernel ridge regression.

\begin{definition}[Kernel ridge estimator]
Given 
$n$
training points \(X = \{\vx_i\}_{i=1}^n\) with 
$n$
labels\footnote{Another approach would be to interpolate the ground-truth labels \(\vy_i\). Since we aim to reconstruct a given model, rather than to simply create the most accurate kernel model in its own right, we interpolate the original model outputs instead.}
\(Y = \{f(\vx_i; \theta)\}_{i=1}^n\) and regularization \(\lambda_{\reg} \geq 0\), the kernel ridge estimator is
\begin{equation}
f_K(\cdot) = \sum_{i=1}^n \alpha_i K(\vx_i, \cdot), \quad \valpha = (K_{XX} + \lambda_{\reg} I)^{-1} Y,
\end{equation}
where \([K_{XX}]_{ij} = K(\vx_i, \vx_j)\) is the NTK at \(\theta\).
\end{definition}

To quantify the error \(f - f_K\), we decompose it through a hierarchy of intermediate approximants.

\begin{definition}[Approximation hierarchy]
Let \(\phi(\vx) = \nabla_\theta f(\vx; \theta)\) denote the gradient features and \(\Sigma = \E[\phi(\vx)\phi(\vx)^\top]\) the population covariance. We define three successive approximations to \(f\):
\begin{enumerate}
    \item The \textbf{best RKHS approximant} \(f_K^*(\vx) = \langle \phi(\vx), w^* \rangle\), where \(w^* = \Sigma^{-1} \E[\phi(\vx) f(\vx; \theta)]\) minimizes the population least-squares loss over all linear functions of gradient features.
    \item The \textbf{regularized population approximant} \(f_{K,\lambda}^*(\vx) = \langle \phi(\vx), w_\lambda^* \rangle\), where \(w_\lambda^* = (\Sigma + \lambda I)^{-1} \E[\phi(\vx) f(\vx; \theta)]\) adds regularization for stability.
    \item The \textbf{predictive estimator} \(f_K\) from \eqref{eqn:predictive_mean}, which estimates \(w_\lambda^*\) from \(n\) samples.
\end{enumerate}
\end{definition}

Each step in this hierarchy introduces error, yielding a three-term decomposition.

\begin{proposition}[Error decomposition]
For a test point \(\vx_*\), the approximation error decomposes as
\begin{align*}
f(\vx_*; \theta) - f_K(\vx_*) = \\
\underbrace{[f - f_K^*] (\vx_*)}_{\cE_{\emph{approx}}} 
+ \underbrace{[f_K^* - f_{K,\lambda}^*](\vx_*)}_{\cE_{\emph{bias}}(\lambda)}
+ \underbrace{[f_{K,\lambda}^* - f_K](\vx_*)}_{\cE_{\emph{var}}(n, \lambda)},
\end{align*}
with expected squared error satisfying
\[\E_{\vx_*}[(f - f_K)^2]^{1/2} \leq \cE_{\emph{approx}} + \cE_{\emph{bias}}(\lambda) + \cE_{\emph{var}}(n, \lambda).\]
\end{proposition}

The three error terms admit the following characterization:
\begin{center}
\begin{tabularx}{\columnwidth}{| p{0.18\columnwidth} | X | X |}
    \hline
    \textbf{Term} & \textbf{Interpretation} & \textbf{Controlled by} \\
    \hline 
    \(\cE_{\text{approx}}\) & Irreducible error from nonlinearity of \(f\) in its gradient features & Network width, proximity to lazy regime \\
    \hline
    \(\cE_{\text{bias}}(\lambda)\) & Shrinkage from regularization; scales as \(O(\lambda^r)\) if \(w^* \in \text{Range}(\Sigma^r)\) & \(\lambda\), spectral alignment of \(f_K^*\) \\
    \hline
    \(\cE_{\text{var}}(n, \lambda)\) & Finite-sample estimation error; scales as \(O(\sqrt{d(\lambda)/n})\) & Sample size \(n\), effective dimension \(d(\lambda) = \sum_j \frac{\mu_j}{\mu_j + \lambda}\) \\
    \hline
\end{tabularx}
\end{center}

Notionally, the bounds are small when the following properties are met:
\begin{itemize}
    \item[\textbf{[1]}] The network operates near the lazy training regime, where \(f(\vx; \theta) \approx f(\vx; \theta_0) + \langle \phi(\vx), \theta - \theta_0 \rangle_\Theta\) (the difference term is of the differential form in Eqn. \eqref{eqn:param-tangent}, ensuring \(\e_{\text{approx}} \approx 0\)).\footnote{Chizat et al. show that lazy training does not require overparameterized networks and bound the distance between lazy and linearized optimization paths \cite{chizat-lazy-2020}. Though this analysis is highly relevant when approximating changes in parameter states, our work assumes a fixed parameter state.}
    \item[\textbf{[2]}] The eigenvalues \(\lambda_j\) of \(K_{XX}\) decay rapidly, whereby the truncation dimension \(r_g(\lambda_\reg) = \sum_j \frac{\lambda_j}{\lambda_j + \lambda_\reg}\) is small.
    \item[\textbf{[3]}] The training labels \(Y\) concentrate on eigendirections of \(K_{XX}\) with large eigenvalues, reducing bias from regularization.
\end{itemize}

\subsection{Approximation Error}\label{sec:ker-approx-error-deriv}
\paragraph{Setup.} 
Let \(f(\cdot; \theta): \cX \to \R\) be a neural network, \(X = \{\vx_i\}_{i=1}^n \sim \cP\) a dataset taken over a distribution, and \(\phi(\vx) = \nabla_\theta f(\vx; \theta) \in \R^P\) the gradient features. 
Assume without loss of generality that \(\E[f(\vx; \theta)] = 0\). 
This can always be achieved by subtracting the population mean, which 
does not affect the gradient features \(\phi(\vx) = \nabla_\theta f(\vx; \theta)\).
Define
\begin{itemize}
\item The kernel as \(K(\vx, \vx') = \langle \phi(\vx), \phi(\vx') \rangle\)
\item The population covariance: \(\Sigma = \E_{\vx \sim \mathcal{D}} [\phi(\vx) \phi(\vx)^\top]\)
\item Empirical covariance: \(\hat{\Sigma} = \frac{1}{n} \Phi^\top \Phi\)
\end{itemize}

\paragraph{Approximation hierarchy.} We construct three successive approximations:
\begin{enumerate}
\item The best RKHS approximant (unregularized, population)
\begin{equation}\label{eqn:best-rkhs}
    f_K^*(\vx) = \langle \phi(\vx), w^* \rangle, \quad w^* = \Sigma^{-1} \E[\phi(\vx) f(\vx; \theta)]
\end{equation}
is the optimal linear predictor in gradient features, achievable with infinite data and no regularization. There is some distance between \(f_K^*\) and our target \(f\).
\item The regularized population approximant
\[f_{K,\lambda}^*(\vx) = \langle \phi(\vx), w_{\lambda}^* \rangle, \quad w_{\lambda}^* = (\Sigma + \lambda I)^{-1} \E[\phi(\vx) f(\vx; \theta)]\]
adds regularization \(\lambda\) to stabilize the solution. There is some distance between \(f_{K, \lambda}^*\) and \(f_K^*\).
\item The empirical kernel ridge estimator
\[f_K(\vx) = \langle \phi(\vx), \hat{w}_{\lambda} \rangle, \quad \hat{w}_{\lambda} = (\hat{\Sigma} + \lambda I)^{-1} \frac{1}{n} \Phi^\top Y\]
uses only \(n\) samples to estimate \(w_{\lambda}^*\). There is some distance between \(f_K\) and \(f_{K, \lambda}^*\).
\end{enumerate}

\paragraph{Error decomposition.} For a test point \(\vx_*\),
\begin{equation}\label{eqn:error-decomp}
f(\vx_*; \theta) - f_K(\vx_*) 
= \underbrace{[f(\vx_*) - f_K^*(\vx_*)]}_{\mathcal{E}_{\text{approx}}} 
+ \underbrace{[f_K^*(\vx_*) - f_*{K,\lambda}^*(\vx_*)]}_{\mathcal{E}_{\text{bias}}(\lambda)} 
+ \underbrace{[f_{K,\lambda}^*(\vx_*) - f_K(\vx_*)]}_{\mathcal{E}_{\text{var}}(n, \lambda)}
\end{equation}
where
\begin{itemize}
\item \(\mathcal{E}_{\text{approx}}\) is the error from network nonlinearity
\item \(\mathcal{E}_{\text{bias}}\) is the bias from regularization shrinkage
\item \(\mathcal{E}_{\text{var}}\) is the variance from finite sampling
\end{itemize}

\paragraph{Approximation gap.}
The optimal \(w^*\) from \eqref{eqn:best-rkhs} minimizes the least-squares \(\E[(f(\vx; \theta) - \langle \phi(\vx), w \rangle)^2]\). Setting the gradient of this expression to zero gives
\[\E[\phi(\vx) \phi(\vx)^\top] w^* = \Sigma w^* = \E[\phi(\vx)f(\vx; \theta)].\]
From here, the minimal residual variance is 
\[\cE^2_{\text{approx}} = \E[f(x; \theta)^2] - 2(w^*)^\top \E[\phi(\vx) f(\vx; \theta)] + (w^*)^\top \Sigma w^*,\]
where substituting \(\E[\phi(\vx) \phi(\vx)^\top] w^* = \Sigma w^*\) gives
\(\cE^2_{\text{approx}} = \E[f(x; \theta)^2] - (w^*)^\top \Sigma w^*\).
This is equivalent to \((1-R^2) \cdot \text{Var}(f)\), where \(R^2\) is the usual coefficient of determination.

Jacot et al. \cite{jacot_ntk_2018} show that, for a network of width \(d\) that has moved a distance of \(\|\Delta \theta\| = \| \theta - \theta_0 \|\), from its \(\cH_K\) parameterization,
\[\cE_{\text{approx}} = (1-R^2) \cdot \text{Var}(f) \sim O \left( \frac{\|\Delta \theta\|^2}{\sqrt{d}}\right) \cdot \sqrt{\E[||H(\vx)||_F^2]},\]
where \(H(\vx) = \nabla_\theta^2 f(x; \theta)\) is the parameter Hessian.

\paragraph{Bias error.}
The regularized population solution satisfies 
\[(\Sigma + \lambda I)w^*_{\lambda} = \E[\phi(\vx) f(\vx; \theta)] = \Sigma w^*.\]

Expand \(w^*\) in its eigenbasis of \(\Sigma\) by letting \(w^* = \sum_j \beta_j \vv_j\), where \(\beta_j = \langle \vv_j, w^* \rangle\). Then,
\[w^*\lambda = \sum_j \frac{\mu_j}{\mu_j + \lambda} \beta_j \vv_j.\]
The bias at a test point \(\vx_*\) is given by 
\[f_K^*(\vx_*) - f_{K, \lambda}^*(\vx_*) = \phi(\vx_*)^\top (w^* - w^*_{\lambda}) = \sum_j \frac{\lambda}{\mu_j + \lambda} \beta_j \langle \phi(\vx_*), \vv_j \rangle. \]
Taking the expectation of this term over \(\vx_* \sim \cP\) and noting that \(\E[\langle \phi(\vx_*), \vv_j \rangle \langle \phi(\vx_*), \vv_k \rangle] = \vv_j^\top \Sigma \vv_k\), we have
\[\cE^2_{\text{bias}} = \E_{\vx_*} \left[ \left( \sum_j \frac{\lambda}{\mu_j + \lambda} \beta_j \langle \phi(\vx_*), \vv_j \rangle \right)^2 \right] 
= \sum_j \left( \frac{\lambda}{\mu_j + \lambda} \right)^2 \mu_j \beta_j^2.\]

Suppose that there exists some integer \(r > 0\) such that \(w^* \in \text{Range}(\Sigma^r)\).
This means \(w^* = \Sigma^r \xi\) for some \(\xi\) with \(|\xi|^2 < \infty\).
Then, \[\mathcal{E}_{\text{bias}}^2 = \sum_j \left(\frac{\lambda}{\mu_j + \lambda}\right)^2 \mu_j^{2r+1} \xi_j^2 \leq \lambda^{2r} |\xi|^2,\]
whereby we conclude
\(\mathcal{E}_{\text{bias}} \sim O(\lambda^r).\)

\paragraph{Variance error.}
The finite-sample estimator is given by
\(\hat{w}_\lambda = (\hat{\Sigma} + \lambda I)^{-1} \frac{1}{n} \Phi^\top Y\). 
Substituting the residual form \(Y = \Phi w^* + r\) gives
\begin{align}
\hat{w}_\lambda &= \underbrace{\left(\hat{\Sigma} + \lambda I \right)^{-1} \frac{1}{n} \hat{\Sigma} w^*}_{\text{signal weight}} + \underbrace{\left(\hat{\Sigma} + \lambda I \right)^{-1} \frac{1}{n} \Phi^\top r.}_{\text{residual noise}}
\end{align}
Supposing that the covariance of \(r\) can be approximated as \(\sigma^2 I\), where \(\sigma^2 = \cE^2_{\text{approx}}\), we get that
\[
\text{Var}(f_K(\vx_*) | X) = \phi(\vx_*)^\top (\hat{\Sigma} + 
\lambda I)^{-1} \left( \frac{\sigma^2}{n} \hat{\Sigma} \right) 
(\hat{\Sigma} + \lambda I)^{-1} \phi(\vx_*).
\]
Averaging this quantity over test points and using \(\E_{\vx_*} [\phi(\vx) \phi(\vx_*)^\top] = \Sigma\) gives
\[
\E_{\vx_*} [\text{Var} (f_K(\vx_*) | X)] = \frac{\sigma^2}{n} \tr \left( \Sigma (\hat{\Sigma} + \lambda I)^{-1} \hat{\Sigma} (\hat{\Sigma} + \lambda I)^{-1}\right).
\]
For large enough \(n\), \(\hat{\Sigma} \approx \Sigma\). Substituting and using spectral decomposition gives
\[\tr(\Sigma(\Sigma + \lambda I)^{-1} \Sigma (\Sigma + \lambda I)^{-1}) = \sum_j \frac{\mu_j^3}{(\mu_j + \lambda)^2}.\]
Since \(\frac{\mu_j}{\mu_j + \lambda} \leq 1\) and \(\mu_j \leq \mu_1\),
\[
\sum_j \frac{\mu_j^3}{(\mu_j + \lambda)^2} \leq \sum_j \frac{\mu_j^2}{\mu_j + \lambda} 
\leq \mu_1 \cdot d(\lambda),
\]
where \(d(\lambda) = \sum_j \frac{\mu_j}{\mu_j + \lambda}\) is the truncation dimension.

Assuming that \(\mu_1 = O(1)\), this gives the final bound 
\[\cE^2_{\text{var}} \leq \frac{\sigma^2 d(\lambda)}{n}.\]
If \(\mu_j \sim j^{-\alpha}\) for some positive \(\alpha\), then \(d(\lambda) = \sum_j \frac{\mu_j}{\mu_j + \lambda} \approx \sum{j: \mu_j > \lambda} 1 \sim \lambda^{-1/\alpha}\), so that \(\cE^2_{\text{var}} \sim \frac{\lambda^{-1/\alpha}}{n}\).

\begin{remark}
The Moore-Aronszajn theorem states that for functions in the RKHS \(\cH_K\) expressed as a linear combination of partial kernel applications \(K(\vx_i, \cdot)\) for some \(\vx_i\), 
\[\sup_{p \geq 0} \left \| \sum_{i=n}^{n+p} \alpha_i K(\vx_i, \cdot) \right \|_{\cH_K} \to 0 \quad \text{ as } n \to \infty.\]
That is, a finite linear combination of some choice of data points and coefficients can approximate the target function to arbitrary precision. For finite-width networks, however, \(\cH_K \subseteq \R^P\) is finite-dimensional, so any \(f \in \cH_K\) admits an exact finite representation, and truncation error vanishes identically. The relevant finite-sample limitation is instead the restriction to \(\text{span}{\phi(\vx_i)}_{i=1}^n\), which is captured by \(\mathcal{E}_{\text{var}}\) in our decomposition.
\end{remark}

\subsection{Projection error}\label{sec:dis-proj-error}
Approximating \(f_K\) by a projected and distilled version of itself introduces two new error terms (\(\cE_{\text{dis}}\) and \(\cE_{\text{proj}}\)) to Eqn. \eqref{eqn:error-decomp}. Here we treat \(\cE_{\text{proj}}\) (\(\cE_{\text{dis}}\) is given by the misalignment error in Proposition~\ref{prop:constrained_pca_decomp}).

The kernel ridge estimator is given by \(f_K(\vx_*) = K_{*X}(K_{XX}+ \lambda I)^{-1}Y\), where \(K_{*X}\) is the kernel between test poing \(\vx_*\) and training set \(X\). With JL projection, we get \(\tf_K(\vx_*) = \tK_{*X} (\tK_{XX} + \lambda I)^{-1} Y\).

The difference between \(f_K(\vx)\) and \(\tf_K(\vx_*)\) decomposes as
\[\cE_\JL 
= \underbrace{(K_{*X} - \tK_{*X})(\tK_{XX} + \lambda I)^{-1}Y}_{\text{(I) test kernel error}}
+ \underbrace{K_{*X} \left[ (K_{XX} + \lambda I)^{-1} - (\tK_{XX} + \lambda I)^{-1} \right]Y}_{\text{(II) inverse perturbation error}}.\]

Given a random projection \(R \in \R^{k \times P}\) with normalized entries, each kernel entry satisfies \(\E[\tK_{ij}] = K_{ij}\) (i.e., the projection is unbiased), and \(|\tK_{ij} - K_{ij}| \leq \e_\JL \cdot |K_{ij}|\), where \(\e_\JL = O(\sqrt{\log(n)/k})\) for projection dimension \(k\).

We can then bound the two error terms:
\begin{enumerate}
    \item[(I)] Using matrix norm bounds, 
    \[\| (K_{*X} - \tK_{*X})(\tK_{XX} + \lambda I)^{-1} Y \|
    \leq \frac{\e_\JL \|K_{*X}\| \cdot \|Y\|}{\mu_{\min} (\tK_{XX}) + \lambda},\]
    where \(\lambda\) is the regularization parameter 
    and \(\mu_{\min} (\tK_{XX})\) 
    is the minimum eigenvalue of \(\tK_{XX}\).
    \item[(II)] Using the identity \(A^{-1} - B^{-1} = A^{-1}(B-A)B^{-1}\), we have
    \begin{align*}
    \left\| (K_{XX} + \lambda I)^{-1} - (\tK_{XX} + \lambda I)^{-1} \right\|_2 \\
    = \left\| (K_{XX} + \lambda I)^{-1} \left[(K_{XX} + \lambda I) - (\tK_{XX} + \lambda I) \right](\tK_{XX} + \lambda I)^{-1}\right\|_2 \\
    \leq \frac{\e_\JL \cdot \|K_{XX}\|_2}{\lambda^2}
    \end{align*}
    This then propagates to 
    \[
    \left\| 
    K_{*X} \left[ (K_{XX} + \lambda I)^{-1} - (\tK_{XX} + \lambda I)^{-1} Y \right]
    \right\|
    \leq \frac{\e_\JL \cdot \|K_{*X}\| \cdot \|K_{XX}\|_2 \cdot \|Y\|}{\lambda^2}.
    \]
\end{enumerate}
Putting these together gives
\[
\cE_\JL \leq \e_\JL \cdot \|K_{*X}\| \cdot \|Y\| \cdot \left( \frac{1}{\lambda} + \frac{\|K_{XX}\|_2}{\lambda^2} \right).
\]
If \(\mu_1\) is the largest eigenvalue of \(K_{XX}\), then
\(\cE_\JL = O(\e_\JL \cdot \frac{\mu_1}{\lambda^2})\).

\begin{remark}
    This is a very conservative upper bound; the actual projection error can be much smaller if the kernel has low truncation rank \(r \ll n\). In this case, the JL error depends on \(r\) rather than \(n\), since only \(r\) directions carry significant variance. We study parameter redundancy in Sec. \ref{sec:jl-redundancy}.
\end{remark}

\subsection{JL parameter redundancy}\label{sec:jl-redundancy}
\begin{proposition}[Parameter redundancy with JL projection] 
Let \(\tK\) be the JL-approximated kernel with projection dimension \(k \geq \frac{c \ln(n)}{\e^2_\JL}\). Let \(\{\tlambda_i\}\) be the eigenvalues of \(\tK\) with mean \(\bar{\tlambda}\).
If \(\tK\) has truncation rank \(r\) at threshold \((1 -\delta)\),
then the original parameters are \((P/r, \e')\)-parameter redundant, 
where \(\e' \leq O \left(\delta \tlambda_{r+1} / \bar{\tlambda} + \e_\JL \right)\).
\end{proposition}

\paragraph{Optimal subspace.} In the terms of Def. \ref{def:param-redundancy}, the optimal choice of \(V\) is the span of the top \(r\) right singular vectors of \(\Phi_X\). Let \(\Phi_X = U \Sigma V^\top\) be the SVD of \(\Phi_X\). Taking \(V_r\) to be the first \(r\) columns of \(V\), the projection of \(\Phi_X\) onto the span of \(V_r\) is
\[\Pi_{V_r}(\Phi_X) = \Phi(X) V_r V_r^\top = U_r \Sigma_r V_r^\top, \]
so that
\[\Pi_{V_r}(\Phi_X) \Pi_{V_r} (\Phi_X)^\top = U_r \Sigma^2_r U_r^\top = \sum_{i=1}^r \lambda_i (\vu_i \otimes \vu_i),\]
where \(\{\lambda_i\}_{i=1}^n\) are the eigenvalues of \(K\).

\paragraph{Numerator error.} The error in the numerator is given by
\begin{equation}\label{eqn:numerator-error}
K - \Pi_{V_r}(\Phi_X) \Pi_{V_r} (\Phi_X)^\top = \sum_{i>r}^n \lambda_i (\vu_i \otimes \vu_i) \implies ||K - \Pi_{V_r}(\Phi_X) \Pi_{V_r} (\Phi_X)^\top||_F^2 = \sum_{i>r}^n \lambda_i^2,
\end{equation}
since \(\vu_i\) is orthonormal to \(\vu_j\) for all \(i \neq j\).

The truncation rank \(r\) at threshold \((1-\delta)\) satisfies
\[\sum_{i>r}^n \lambda_i\ \leq \delta \sum_{i=1}^n \lambda_i = \delta \cdot \tr(K).\]
Furthermore, we have that 
\begin{equation}\label{eqn:eig-square-bound}
\sum_{i>r} \lambda_i^2 \leq \lambda_{r+1} \sum_{i>r} \lambda_i \leq \lambda_{r+1} \cdot \delta \cdot \tr(K).
\end{equation}

\paragraph{Denominator error.} Cauchy-Schwartz gives
\begin{equation}\label{eqn:eig-cs}
    ||K||_F^2 = \sum_{i=1}^n \lambda_i^2 \geq \frac{\tr(K)^2}{n}.
\end{equation}

\paragraph{Combined error.} (\ref{eqn:numerator-error}), (\ref{eqn:eig-square-bound}), and (\ref{eqn:eig-cs}) give
\begin{equation}
    \frac{||K  - K_R||_F^2}{||K||_F^2} \leq \frac{\lambda_{r+1} \cdot \delta \cdot \tr(K)}{\tr(K)^2 / n} = \frac{n \delta \lambda_{r+1}}{\tr(K)} = \delta \cdot \frac{\lambda_{r+1}}{\bar{\lambda}},
\end{equation}
where \(\bar{\lambda} = \tr(K) / n\) is the mean eigenvalue. Choosing a \(\delta\) factor with an truncation rank whose eigenvalue is less than the mean is a practical way to further scale down the error.

\paragraph{JL approximation.} We don't observe \(K\) or its spectrum directly, but JL guarantees that the projected \(\tK\) has eigenvalues \(\tilde{\lambda}_i\) satisfying
\[(1 - \e_\JL)^2 \lambda_i \leq \tlambda_i \leq (1 + \e_\JL)^2 \lambda_i\]
with high probability. Therefore, the truncation rank \(\tilde{r}\) of \(\tK\) approximates that of \(K\), as does the ratio \(\tlambda_{r+1} / \bar{\tlambda}\), giving an error bound of \(O\left(\delta \frac{\tlambda_{r+1}}{\bar{\tlambda}} + \e_\JL \right)\).

\section{Proofs for Section~\ref{sec:dd_as_pca}}
\label{appsec:pca_proof}

\paragraph{Conventions.}
Throughout, matrices \(\Phi\in\R^{n \times P}\) have \emph{rows in parameter space} \(\R^P\) (e.g.\ per-logit
\(\Phi=\Phi_X^c\in\R^{n\times P}\), or a stacked multi-logit matrix). For a parameter-space subspace
\(V\subset\R^P\) with orthogonal projector \(\Pi_V\in\R^{P\times P}\), projecting each row of \(\Phi\) onto \(V\)
corresponds to \emph{right}-multiplication: \(\Phi\Pi_V\), with residual \(\Phi(I-\Pi_V)\).

\subsection{Preliminary lemma: reconciling loss gradients with logit-gradient features}

\begin{lemma}[Chain rule: loss gradients lie in the logit-gradient span]\label{lem:chain_rule_span}
Fix \(\theta\) and logits \(f(\vx;\theta)\in\R^C\). For each \(c\in[C]\) define the logit-gradient feature
\(\phi^c(\vx):=\nabla_\theta f^c(\vx;\theta)\in\R^P\). Let \(\ell:\R^C\times\cY\to\R\) be any per-example loss,
and define the \emph{logit sensitivity}
\[
\delta(\vx,y;\theta):=\nabla_z \ell(z,y)\vert_{z=f(\vx;\theta)}\in\R^C.
\]
Then for any labeled example \((\vx, \vy)\),
\[
\nabla_\theta \ell(f(\vx;\theta), \vy)
= \sum_{c=1}^C \delta_c(\vx, \vy;\theta)\,\phi^c(\vx).
\]
Consequently, for a distilled dataset \(\tD=(\tX,\tilde Y)\) with \(|\tX|=m\), define
the stacked logit-gradient matrix
\[
\tPhi :=
\begin{bmatrix}
\Phi_{\tX}^1\\ \vdots\\ \Phi_{\tX}^C
\end{bmatrix}
\in \R^{(mC)\times P},
\qquad
[\Phi_{\tX}^c]_i:=\phi^c(\vtx_i).
\]
Let \(\tdelta(\theta)\in\R^{mC}\) stack \(\delta(\vtx_i,\vty_i;\theta)\) over \(i\).
Then the distilled loss gradient satisfies
\[
g_{\tD}(\theta):=\nabla_\theta \mathcal L_{\tD}(\theta)
=\tPhi^\top \tdelta(\theta)\in \mathrm{colspan}(\tPhi^\top)=:V(\tD).
\]
\end{lemma}

\noindent
Lemma~\ref{lem:chain_rule_span} is the key reconciliation used in Sec.~\ref{sec:dd_as_pca}: although DD is
written in terms of \emph{loss} gradients, those gradients always lie in the span of \emph{logit} gradients
\(\nabla_\theta f^c(\tX;\theta)\), which are exactly the kernel features.

\subsection{Proof of Theorem~\ref{thm:dd_proj_final} (projection residual controls one-step progress)}
\label{appsec:thm-proof}
\begin{proof}[Proof of Theorem~\ref{thm:dd_proj_final}]
Fix \(t\) and abbreviate \(\mathcal L_t(\theta)\) by \(\mathcal L(\theta)\) and \(g_t:=\nabla_\theta\mathcal L(\theta)\).
Since \(\mathcal L\) is \(L\)-smooth, for any update direction \(v\) and step size \(\eta\),
\[
\mathcal L(\theta-\eta v)
\le \mathcal L(\theta)
+\langle g_t, -\eta v\rangle
+\frac{L}{2}\|\eta v\|^2
=
\mathcal L(\theta)-\eta\langle g_t,v\rangle+\frac{L\eta^2}{2}\|v\|^2.
\]
Applying this with \(v=g_{\tD}(\theta)\) yields
\[
\mathcal L_t(\theta^+(\tD))
\le \mathcal L_t(\theta)
-\eta\langle g_t,g_{\tD}(\theta)\rangle
+\frac{L\eta^2}{2}\|g_{\tD}(\theta)\|^2,
\]
and taking expectation over \(t\sim\mathcal T\) gives \eqref{eq:headline_regret_bound}.

For the second claim, fix a subspace \(V:=V(\tD)\) with orthogonal projector \(\Pi_{\tD}\).
Consider the quadratic upper model
\[
q_t(\Delta):=\langle g_t,\Delta\rangle+\frac{L}{2}\|\Delta\|^2,
\qquad \Delta\in V.
\]
Because \(q_t\) is strictly convex, its minimizer over \(\Delta\in V\) is characterized by the first-order condition:
for all \(\Delta'\in V\),
\[
\langle g_t + L\Delta_t^\star,\ \Delta'-\Delta_t^\star\rangle = 0.
\]
Equivalently, \(g_t+L\Delta_t^\star\) is orthogonal to \(V\), i.e.\ \(\Pi_{\tD}(g_t+L\Delta_t^\star)=0\),
which gives
\[
\Delta_t^\star=-\frac{1}{L}\Pi_{\tD}g_t.
\]
By \(L\)-smoothness, \(\mathcal L_t(\theta+\Delta)\le \mathcal L_t(\theta)+q_t(\Delta)\) for all \(\Delta\); hence
\[
\mathcal L_t(\theta)-\mathcal L_t(\theta+\Delta_t^\star)
\ge
-q_t(\Delta_t^\star)
= -\Big\langle g_t,-\tfrac{1}{L}\Pi_{\tD}g_t\Big\rangle
-\frac{L}{2}\Big\|-\tfrac{1}{L}\Pi_{\tD}g_t\Big\|^2
= \frac{1}{2L}\|\Pi_{\tD}g_t\|^2.
\]
Finally, since \(\Pi_{\tD}\) is an orthogonal projector,
\[
\|\Pi_{\tD}g_t\|^2 = \|g_t\|^2-\|(I-\Pi_{\tD})g_t\|^2,
\]
which gives \eqref{eq:proj_residual_controls_progress_final}.
\end{proof}

\subsection{Proof of Corollary~\ref{cor:pca_final} (PCA of gradient covariance is optimal)}
\label{appsec:cor-proof}
\begin{proof}[Proof of Corollary~\ref{cor:pca_final}]
Let \(V\subset\R^P\) be any \(r\)-dimensional subspace with orthogonal projector \(\Pi_V\). Using idempotence and
symmetry of \(\Pi_V\),
\[
\|(I-\Pi_V)g_t\|^2
= g_t^\top(I-\Pi_V)g_t
= \mathrm{tr}\!\big((I-\Pi_V)g_tg_t^\top\big).
\]
Taking expectation over \(t\) and defining \(G:=\E_t[g_tg_t^\top]\) yields
\[
\E_t[\|(I-\Pi_V)g_t\|^2]
= \mathrm{tr}((I-\Pi_V)G)
= \mathrm{tr}(G)-\mathrm{tr}(\Pi_V G).
\]
Thus minimizing \(\E_t[\|(I-\Pi_V)g_t\|^2]\) over \(\dim(V)=r\) is equivalent to maximizing
\(\mathrm{tr}(\Pi_VG)\) over rank-\(r\) orthogonal projectors. By the Ky Fan maximum principle, the maximizer is
the projector onto the top-\(r\) eigenspace of \(G\), attaining value \(\sum_{j\le r}\lambda_j\). Therefore the
minimum residual equals
\[
\mathrm{tr}(G)-\sum_{j\le r}\lambda_j
= \sum_{j>r}\lambda_j.
\]
Moreover, if \(\mathrm{tr}(\Pi_{V^\star}G)-\mathrm{tr}(\Pi_VG)\le\delta\), then
\[
\E_t[\|(I-\Pi_V)g_t\|^2]
= \mathrm{tr}(G)-\mathrm{tr}(\Pi_VG)
= \underbrace{\mathrm{tr}(G)-\mathrm{tr}(\Pi_{V^\star}G)}_{=\sum_{j>r}\lambda_j}
+\Big(\mathrm{tr}(\Pi_{V^\star}G)-\mathrm{tr}(\Pi_VG)\Big)
\le
\sum_{j>r}\lambda_j+\delta,
\]
as claimed.
\end{proof}

\subsection{Proof of Proposition~\ref{prop:constrained_pca_decomp} (feature-space tail + misalignment)}
\label{appsec:prop-proof}
\begin{proof}[Proof of Proposition~\ref{prop:constrained_pca_decomp}]
Let \(\Phi\in\R^{n\times P}\) be a gradient-feature matrix (rows in \(\R^P\)) with SVD \(\Phi=U\Sigma W^\top\).
Let \(W_r\) be the top-\(r\) right singular vectors and \(\Pi^\star:=W_rW_r^\top\) the rank-\(r\) PCA projector.
For any rank-\(r\) orthogonal projector \(\Pi\) on \(\R^P\) (in particular, \(\Pi=\Pi_{\tD}\) when
\(\dim V(\tD)=r\)), using \(\|M\|_F^2=\mathrm{tr}(M^\top M)\) and \(\Pi=\Pi^\top=\Pi^2\),
\[
\|\Phi(I-\Pi)\|_F^2
= \mathrm{tr}\!\big((I-\Pi)\Phi^\top\Phi(I-\Pi)\big)
= \mathrm{tr}(\Phi^\top\Phi)-\mathrm{tr}(\Phi^\top\Phi\,\Pi).
\]
Define \(A:=\Phi^\top\Phi=W\Sigma^2W^\top\), whose eigenvalues are \(\{\sigma_j(\Phi)^2\}_{j=1}^P\).
By the Ky Fan maximum principle, \(\mathrm{tr}(A\Pi)\) is maximized over rank-\(r\) projectors by \(\Pi^\star\),
with \(\mathrm{tr}(A\Pi^\star)=\sum_{j\le r}\sigma_j(\Phi)^2\). Therefore,
\[
\|\Phi(I-\Pi)\|_F^2
= \underbrace{\big(\mathrm{tr}(A)-\mathrm{tr}(A\Pi^\star)\big)}_{=\sum_{j>r}\sigma_j(\Phi)^2\ \text{(PCA tail)}}
+ \underbrace{\big(\mathrm{tr}(A\Pi^\star)-\mathrm{tr}(A\Pi)\big)}_{\ge 0\ \text{(captured-energy gap / misalignment)}},
\]
which is exactly the decomposition stated in Proposition~\ref{prop:constrained_pca_decomp}.

\subsection{Wasserstein metric dataset distillation}\label{appsec:wmdd}
Since directly solving this optimization problem is often prohibitive, it is often convenient to accomplish this through a surrogate objective. Here, we use a modified version of \textbf{Wasserstein metric DD (WMDD)} \cite{liu_wmdd_2025}, which aims to solve the surrogate problem of feature matching, under the assumption that a performant distilled dataset will be \textit{distributionally close to the original dataset in both the data space and the feature space}. To that end, it defines the loss function on the dataset: \[\cL(\tX) = \cL_{\text{feature}} + \lambda_{\text{BN}} \cL_{\text{BN}}, \quad \text{where}\] \begin{itemize} 
    \item \(\cL_{\text{feature}} = \sum_{k,j} \|f_e (\vtx_{k,j}) - b_{k,j}\|2^2\) matches synthetic features given by \(f_e\) to their barycenter targets \(b_{k,j}\), which are computed from a pretrained \(f\). 
    \item \(\cL_{\text{BN}} = \sum_{\ell} \| \mu_\ell^{(k)} - \hat{\mu}_\ell^{(k)}\| + \|\sigma_\ell^{(k)} - \hat{\sigma}_\ell^{(k)}\|\) enforces that per-class batch normalization statistics (mean \(\mu\) and standard deviation \(\sigma\) at each layer \(\ell\)) of the synthetic data match those precomputed from real data. \end{itemize} 
    
The barycenter weights \(w_{k,j}\) are preserved and used during downstream knowledge distillation to weight each synthetic sample's contribution. See \cite{liu_wmdd_2025} for full details.

\end{proof}
\section{Experiments on other datasets}
\label{appsec:other-experiments}

In \Cref{fig:imagewoof}, although the more difficult classification task prevents the same accuracy as ImageNette, performance still quickly saturates to the baseline. Interestingly, ImageWoof exhibits a very different condition number and minimum eigenvalue profile. Furthermore, the ImageNette-Resnet18 pair responds to \Cref{alg:local_global_comp}.

\begin{figure}[!h]
    \centering
    \includegraphics[width=0.45\linewidth]{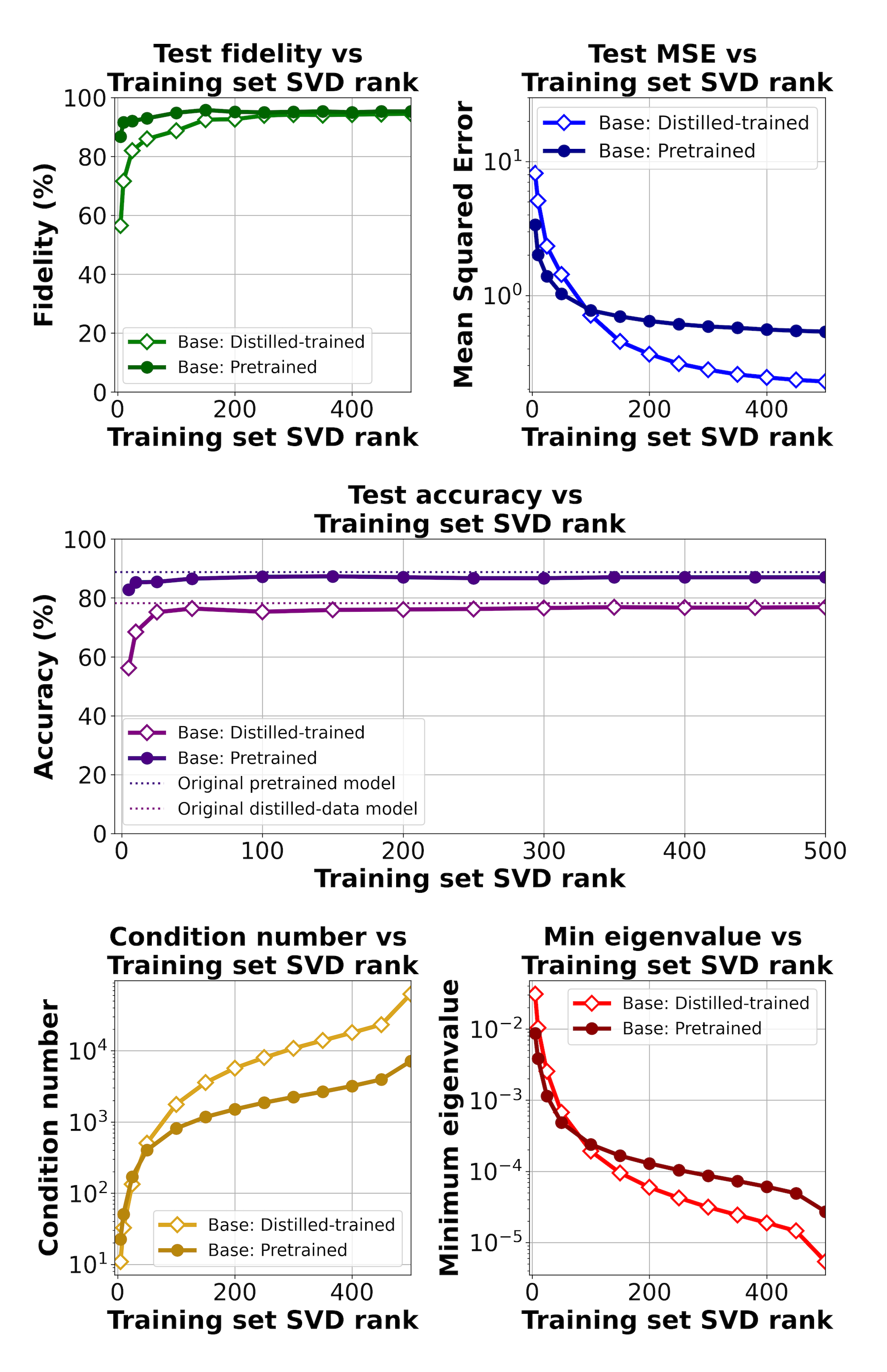}
    \caption{The same measures of accuracy as in \Cref{fig:size-acc-fid-mse} saturate quickly with increasing rank. The best rank-\(r\) approximation is taken by substituting \(U\) and \(\Sigma\) with \(U^{(r)}\) and \(\Sigma^{(r)}\) in \eqref{eqn:woodbury}.}
    \label{fig:rank-acc-fid-mse}
\end{figure}

\begin{figure}[!h]
    \centering
    \includegraphics[width=0.50\linewidth]{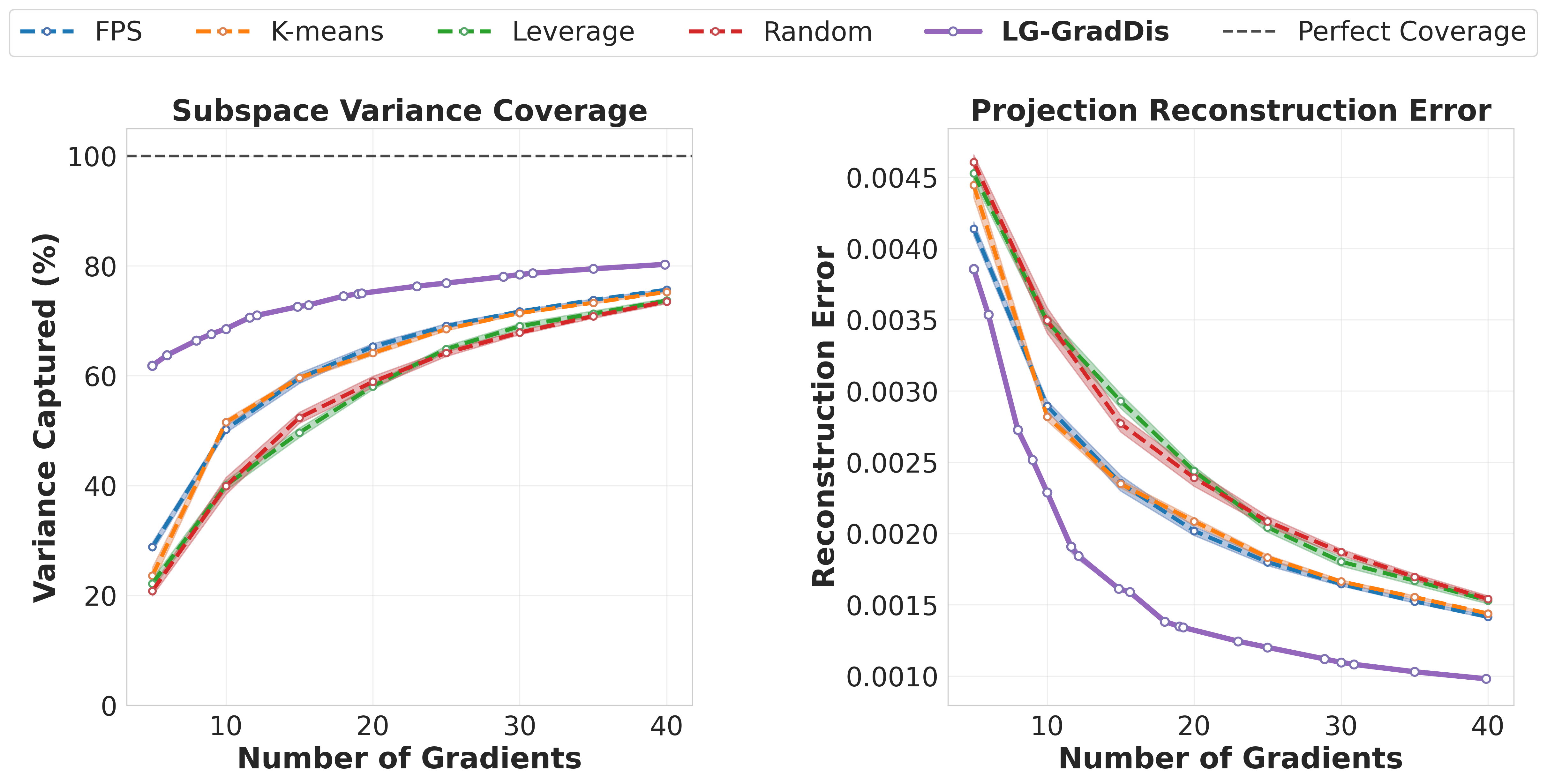}
    \caption{\textbf{Subspace variance coverage} is computed by projecting the centered gradients \(\Phi_X\) onto \(VV^\top\), where \(V \in \R^{P \times m}\) is an orthonormal basis of the distilled subspace computed using QR decomposition. Total variance coverage is measured as the ratio \(\|\Phi_X V V^\top\|_F^2 / \|\Phi_X\|^2\).
    \textbf{Projection reconstruction error} Reconstruction error is then computed by \(\|\Phi_X - \Phi_X VV^\top\|^2 / n\), measuring the information lost when representing the training gradients in the lower-dimensional distilled subspace. All results are taken with \(H=10\) clusters and various thresholds \(\tau_v\) and \(\tau_g\).}
    \label{fig:coverage-plot}
\end{figure}

\begin{figure}
    \centering
    \begin{subfigure}{0.45\linewidth}
        \centering
        \includegraphics[width=\linewidth]{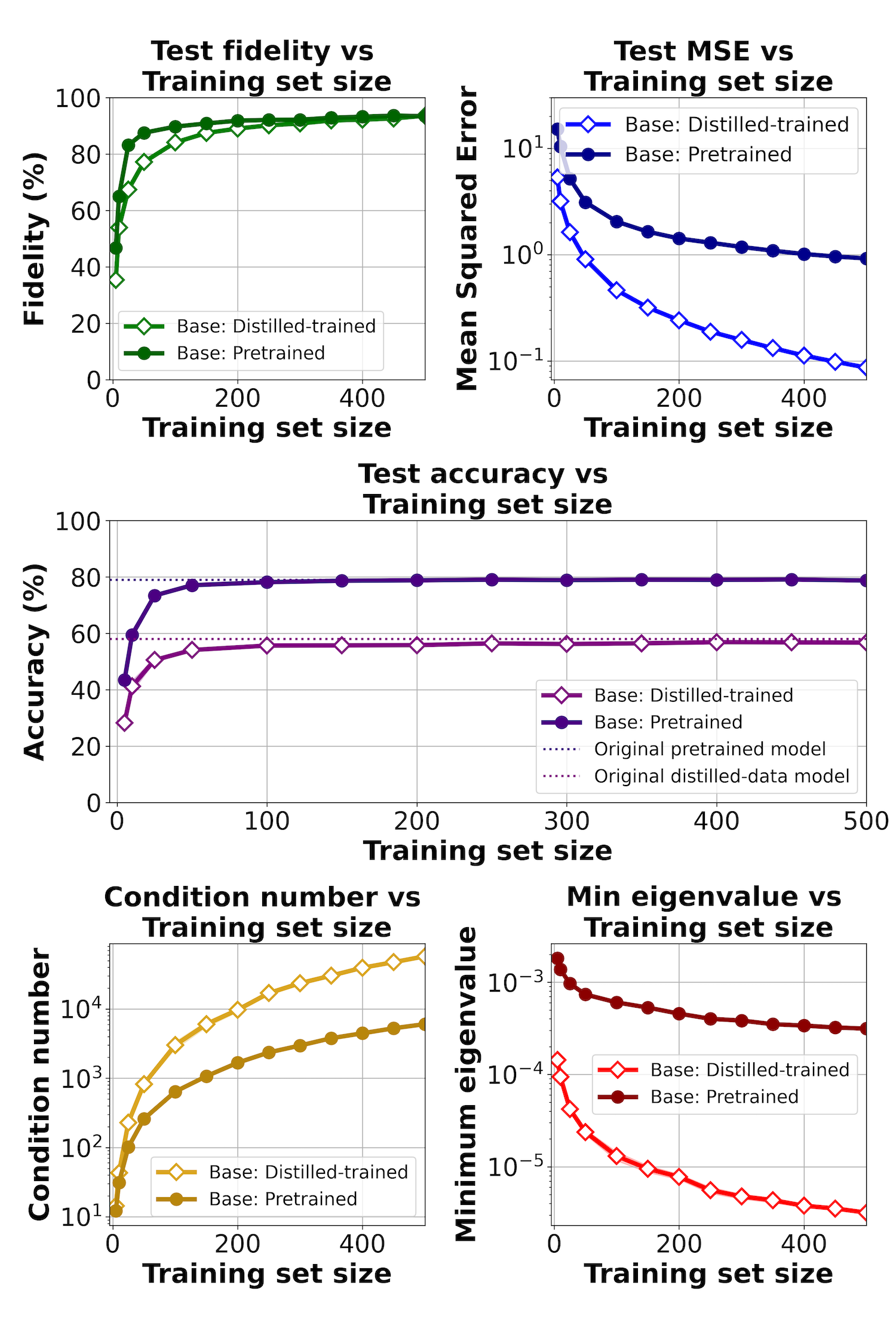}
        \caption{Measures given in \Cref{fig:size-acc-fid-mse} as a function of the number of gradients.}
    \end{subfigure}
    \begin{subfigure}{0.45\linewidth}
        \centering
        \includegraphics[width=\linewidth]{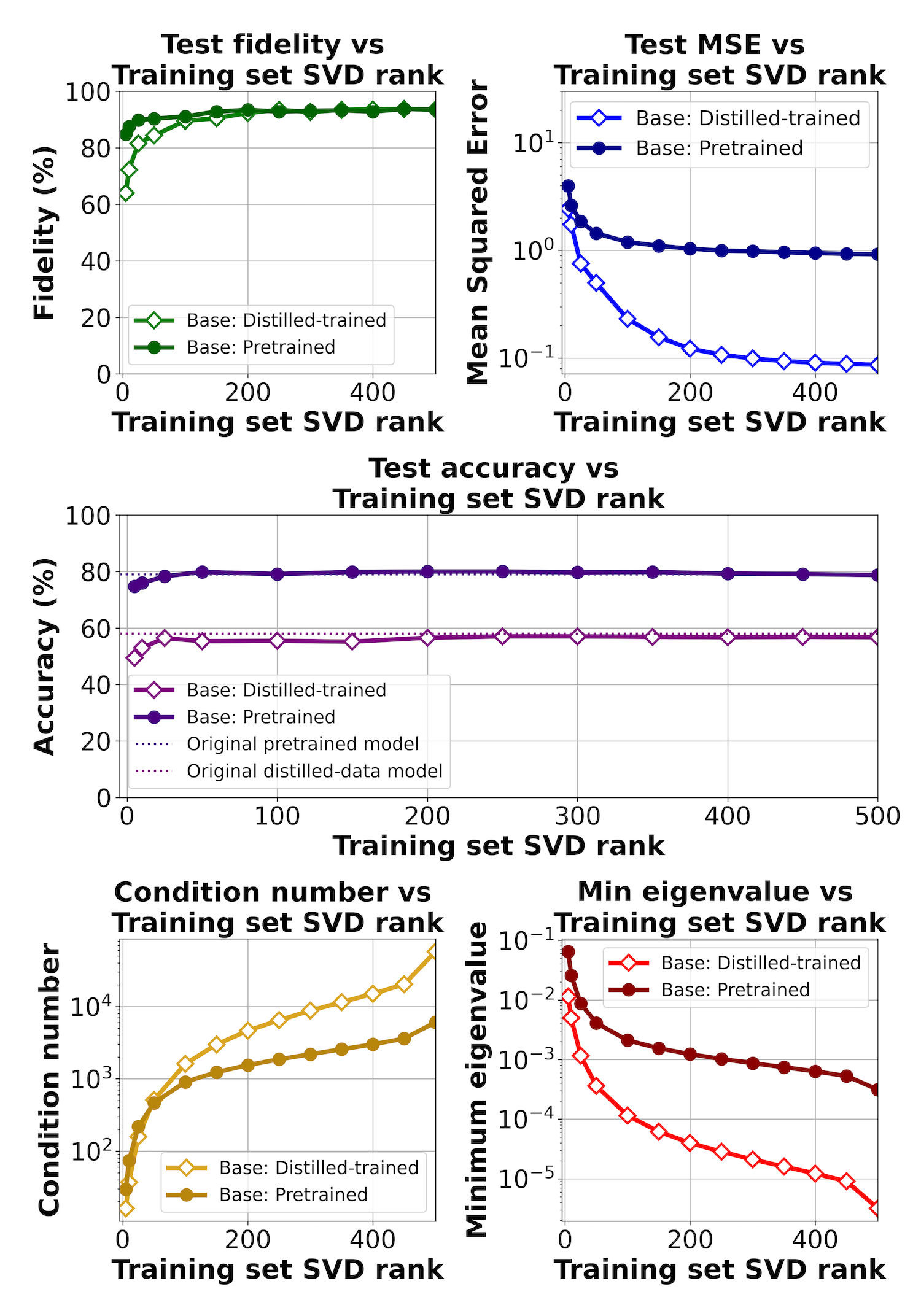}
        \caption{Measures given in \Cref{fig:size-acc-fid-mse} as a function of the rank-\(r\) SVD truncation of the kernel matrix.}
    \end{subfigure}
    \begin{subfigure}{0.6\linewidth}
        \centering
        \includegraphics[width=\linewidth]{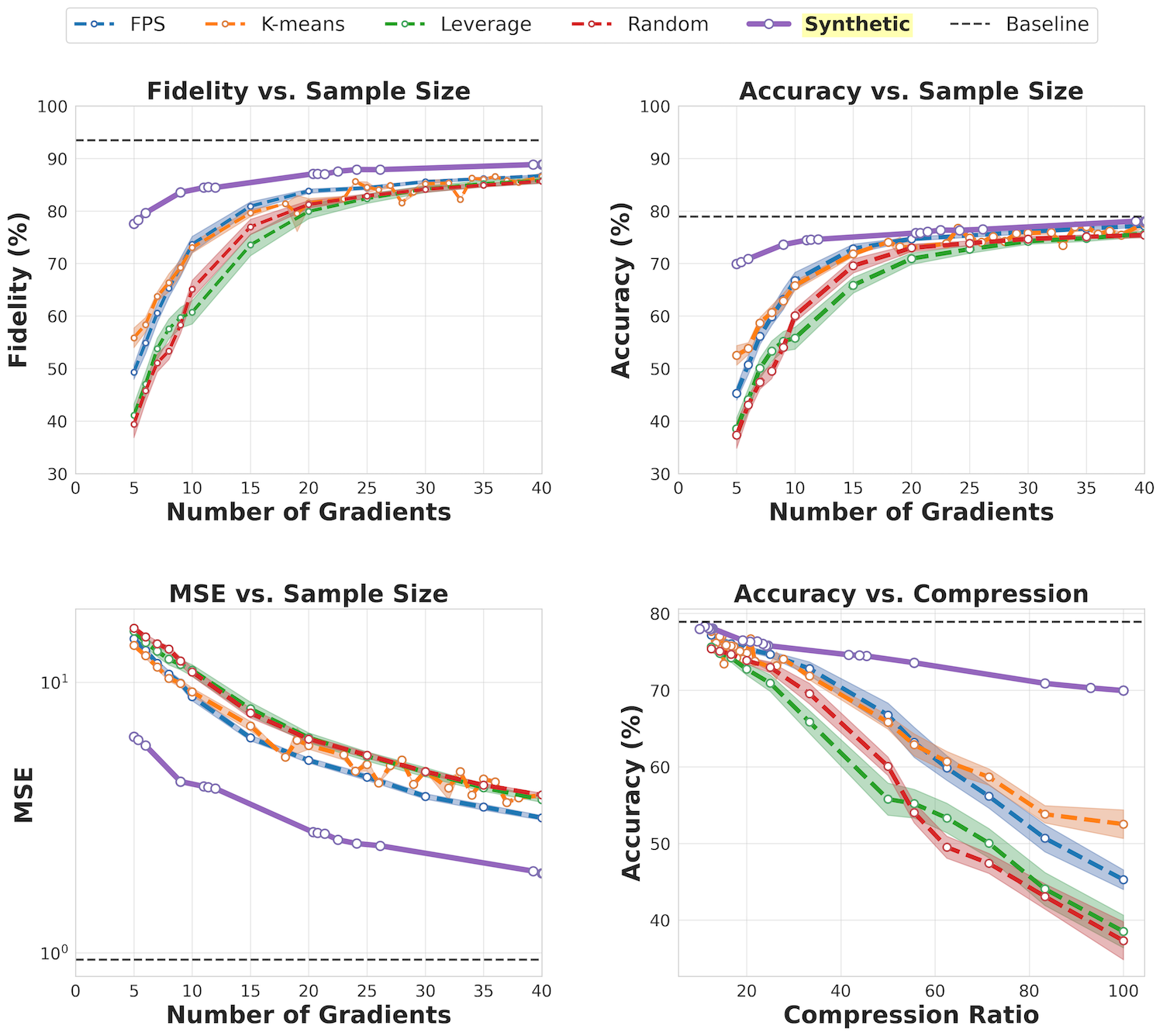}
        \caption{ImageWoof fidelity, accuracy, and MSE under \Cref{alg:local_global_comp}.}
    \end{subfigure}
    \caption{Fidelity, accuracy, and spectral tests on the \text{ImageWoof} dataset with a ResNet-18 model.}
    \label{fig:imagewoof}
\end{figure}

\section{Local-global gradient distillation}
\label{appsec:local_global_comp}

\subsection{Containment and gaps}
\label{appsec:contain-gaps}
\Cref{alg:local_global_comp} is based on the observation that there exists a significant gap between the subspace spanned by individual clusters and the subspace needed for the kernel to make correct predictions. In Fig. \ref{fig:containment_gaps}, we provide the subspace containment results for the other nine classes described in Sec. \ref{sec:local-global}. The same trends are observed –– each cluster attends to a small portion of total PCs in the effective rank –– but different clusters are activated in the first few principal directions for different classes.

\begin{figure}[!h]
\centering
\begin{tabular}{@{}ccc@{}}
\includegraphics[width=0.25\linewidth]{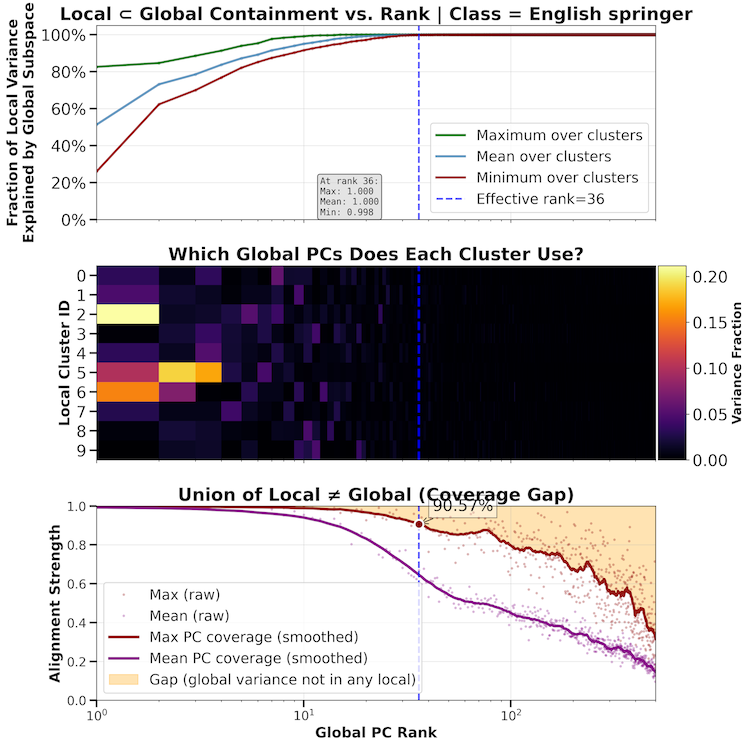} &
\includegraphics[width=0.25\linewidth]{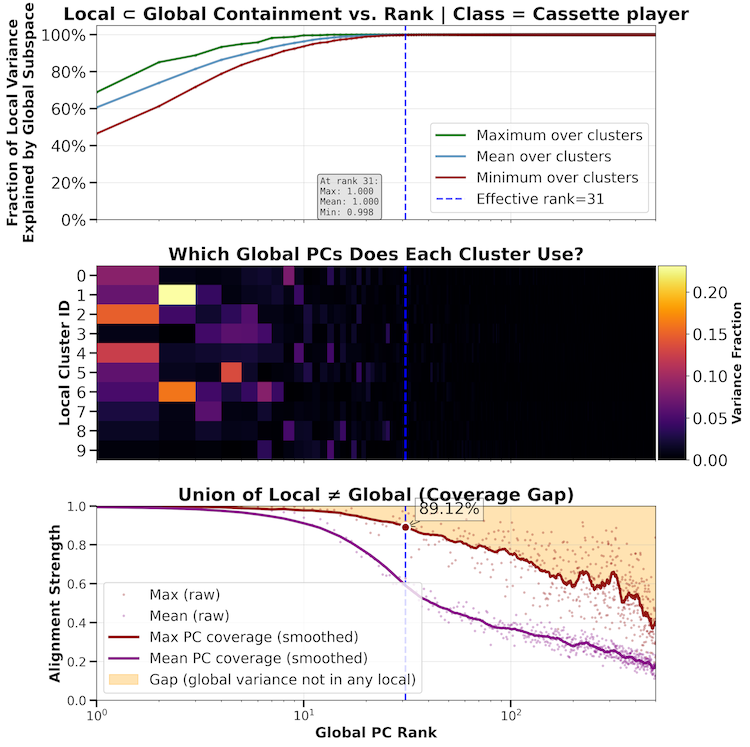} &
\includegraphics[width=0.25\linewidth]{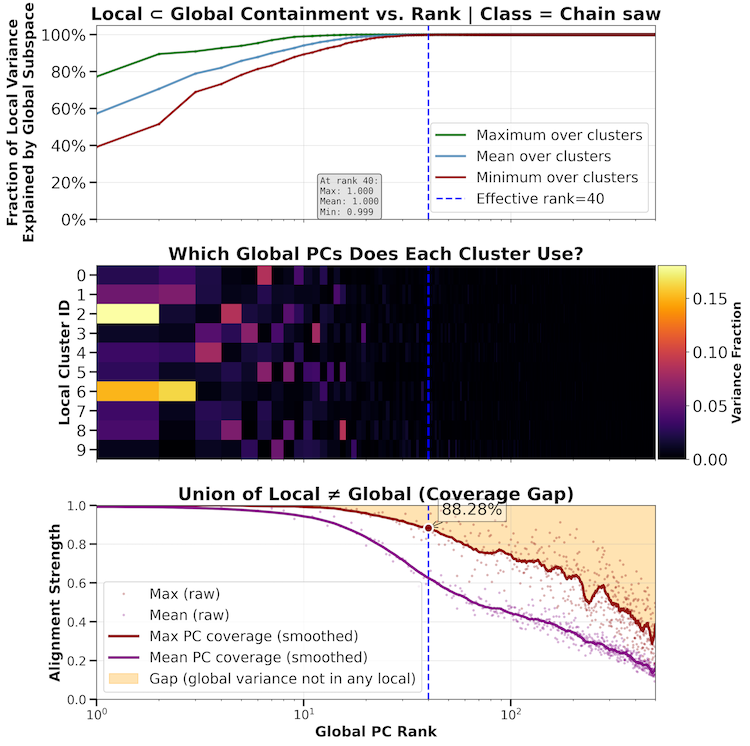} \\[2mm]

\includegraphics[width=0.25\linewidth]{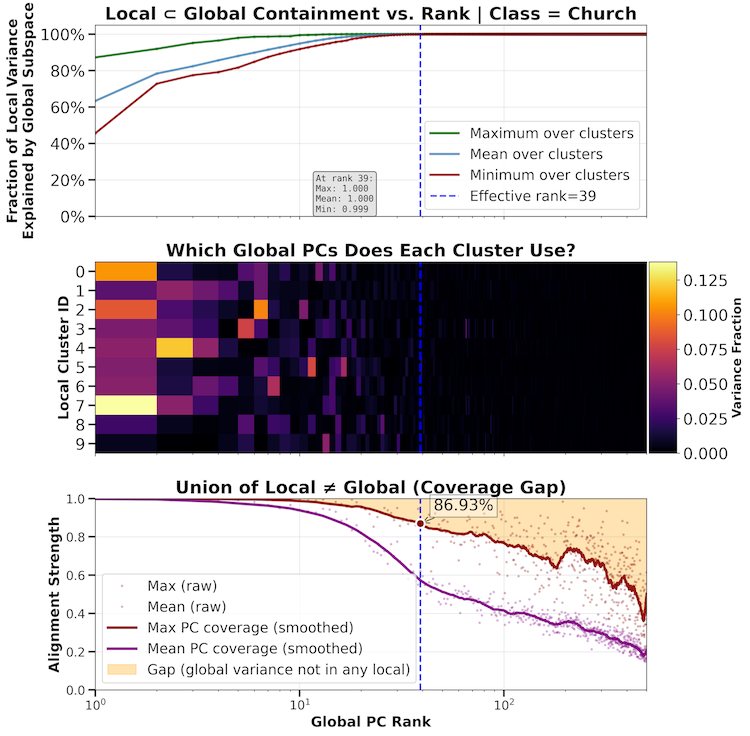} &
\includegraphics[width=0.25\linewidth]{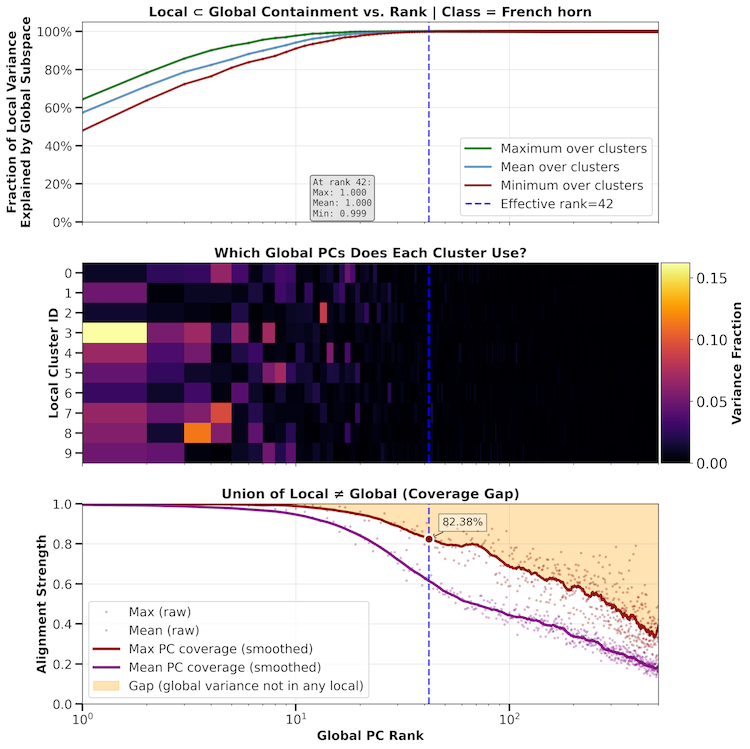} &
\includegraphics[width=0.25\linewidth]{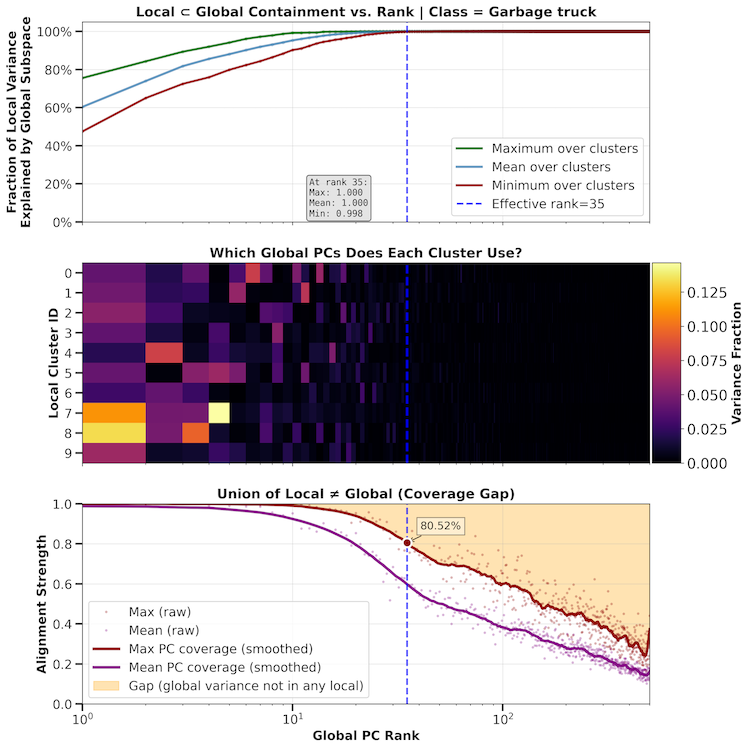} \\[2mm]

\includegraphics[width=0.25\linewidth]{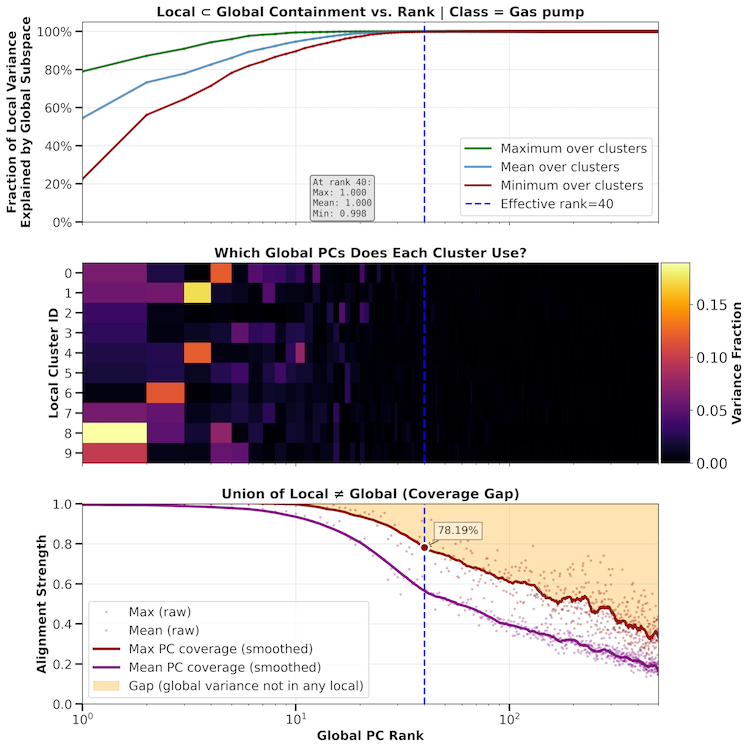} &
\includegraphics[width=0.25\linewidth]{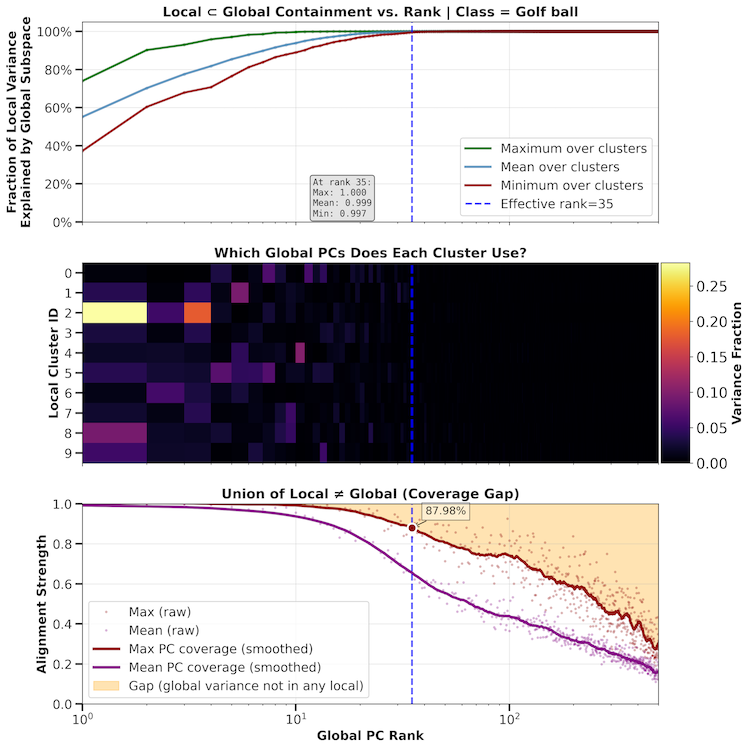} &
\includegraphics[width=0.25\linewidth]{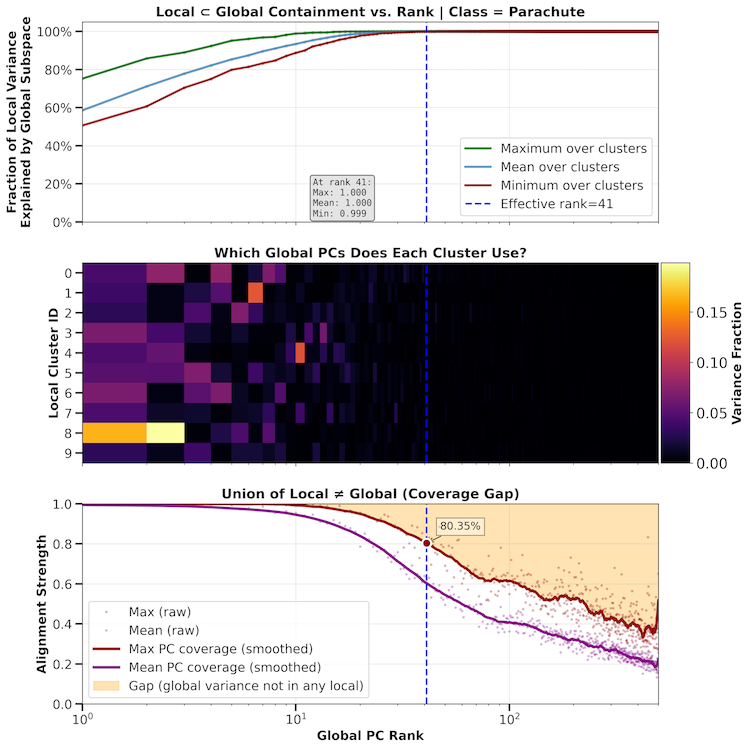}
\end{tabular}
\caption{Explained variance and containment gaps across classes as a function of rank (ImageNette dataset, ResNet-18 model).}
\label{fig:containment_gaps}
\end{figure}

\begin{algorithm}[!h]
\caption{Local-Global Gradient Distillation}
\label{alg:local_global_comp}
\begin{algorithmic}[1]
\REQUIRE Stacked gradients \(\Phi \in \R^{m \times k \times C}\), labels \(Y \in \R^{m \times C}\), cluster count \(H\), variance threshold \(\tau_v\), gap threshold \(\tau_g\)
\ENSURE Synthetic gradients \(\hat{\Phi} \in \R^{s \times k \times C}\), synthetic labels \(\hat{Y} \in \R^{s \times C}\)

\STATEX
\STATEX \textbf{// Step 1: Kernel computation and clustering}
\STATE \(K \gets \frac{1}{k} \Phi \Phi^\top \in \R^{m \times m \times C}\) \COMMENT{Per-class kernel}
\STATE \(\bar{K} \gets \frac{1}{C} \sum_{c=1}^{C} K^{c}\) \COMMENT{Class-averaged kernel}
\STATE \(\{\cI_h\}_{h=1}^{H} \gets \textsc{SpectralCluster}(\bar{K}, H)\) \COMMENT{Cluster indices}

\STATEX
\STATEX \textbf{// Step 2: Global eigendecomposition}
\STATE \(U_g \Sigma_g U_g^\top \gets \textsc{eigendecomp}(\bar{K})\)
\STATE \(r_g \gets \min\{r : \sum_{i=1}^{r} \sigma_i / \sum_{i=1}^{m} \sigma_i \geq \tau_v\}\) \COMMENT{Global effective rank}

\STATEX
\STATEX \textbf{// Step 3: Local eigendecomposition and coverage analysis}
\STATE \(\mathbf{c} = [c_1, \ldots, c_{r_g}] \gets \mathbf{0} \in \R^{r_g}\) \COMMENT{Coverage of each global direction}
\FOR{\(h = 1, \ldots, H\)}
    \STATE \(\bar{K}_h \gets \bar{K}[\cI_h, \cI_h]\) \COMMENT{Local kernel}
    \STATE \(U_h \Sigma_h U_h^\top \gets \textsc{SVD}(\bar{K}_h)\)
    \STATE \(r_h \gets \min\{r : \sum_{i=1}^{r} [\Sigma_h]_{ii} / \text{tr}(\Sigma_h) \geq \tau_v\}\)
    \FOR{\(j = 1, \ldots, r_g\)}
        \STATE \(\vu \gets U_g[\cI_h, j]\) \COMMENT{Global eigenvector restricted to cluster}
        \STATE \(\vu_{\text{proj}} \gets U_h[:, 1{:}r_h] \, (U_h[:, 1{:}r_h]^\top \vu)\) \COMMENT{Project onto local span}
        \STATE \(c_j \gets \|\vu_{\text{proj}}\|^2 / \|\vu\|^2\) \COMMENT{Update coverage}
    \ENDFOR
\ENDFOR
\STATE \(\cG \gets \{j : c_j < \tau_g\}\) \COMMENT{Gap directions}

\STATEX
\STATEX \textbf{// Step 4: Distill local representatives}
\STATE \(\hat{\Phi} \gets [\,], \quad \hat{Y} \gets [\,], \quad \cU \gets [\,]\)
\FOR{\(h = 1, \ldots, H\)}
    \FOR{\(j = 1, \ldots, r_h\)}
        \STATE \(\vu \gets U_h[:, j] / \|U_h[:, j]\|\) \COMMENT{Normalized local eigenvector}
        \STATE \(\hat{\phi} \gets \Phi[\cI_h]^\top \vu \in \R^{k \times C}\) \COMMENT{Synthesize gradient}
        \STATE \(\hat{\mathbf{y}} \gets Y[\cI_h]^\top \vu \in \R^{C}\) \COMMENT{Synthesize label}
        \STATE Append \(\hat{\phi}\) to \(\hat{\Phi}\), \(\hat{\mathbf{y}}\) to \(\hat{Y}\)
        \STATE \(\hat{\vu} \gets \mathbf{0} \in \R^N\); \(\hat{\vu}[\cI_h] \gets \vu\) \COMMENT{Lift to full space}
        \STATE Append \(\hat{\vu}\) to \(\cU\)
    \ENDFOR
\ENDFOR

\STATEX
\STATEX \textbf{// Step 5: Distill gap representatives}
\FOR{\(j \in \cG\)}
    \STATE \(\vv \gets U_g[:, j] / \|U_g[:, j]\|\) \COMMENT{Normalized global eigenvector}
    \STATE \(\hat{\phi} \gets \Phi^\top \vv\) \COMMENT{Synthesize from full set}
    \STATE \(\hat{\mathbf{y}} \gets Y^\top \vv\)
    \STATE Append \(\hat{\phi}\) to \(\hat{\Phi}\), \(\hat{\mathbf{y}}\) to \(\hat{Y}\), \(\vv\) to \(\cU\)
\ENDFOR

\STATEX
\STATEX \textbf{// Step 6: Orthogonalize to remove redundancy}
\STATE \(Q, R \gets \textsc{QR}([\cU])\) \COMMENT{\([\cU] \in \R^{N \times m}\)}
\STATE \(\cS \gets \{i : |R_{ii}| > \e \cdot \max_j |R_{jj}|\}\) \COMMENT{Non-redundant indices}
\STATE \(\hat{\Phi} \gets \hat{\Phi}[\cS], \quad \hat{Y} \gets \hat{Y}[\cS]\)

\STATEX
\STATE \textbf{return} \(\hat{\Phi}, \hat{Y}\)
\end{algorithmic}
\end{algorithm}

\subsection{Complexity and grid search}
\label{appsec:grid-search}
The length of Algorithm \ref{alg:local_global_comp} belies a reasonable complexity profile. Let \(n\) be the number of original gradients, \(H\) the number of clusters, \(r_g\) the global effective rank, and \(\bar{r} = \frac{1}{H}\sum_h r_h\) the average local effective rank. We get the following complexity considerations:
\begin{itemize}
    \item \textbf{Kernel computation} scales as \(O\left(n^2 k C \right)\)
    \item \textbf{Clustering} is \(O\left(n^3 \right)\) for spectral clustering (dominated mostly by eigendecomposition)
    \item \textbf{Global SVD} is \(O(n^3)\)
    \item \textbf{Local SVDs} is \(O \left(\sum_h |\cI_h|^3 \right) = O\left(n^3 / H^2 \right)\), assuming roughly balanced clusters
    \item \textbf{Synthesis} is \(O(m \bar{n} k C)\) where \(m\) is the number of synthetic gradients and \(\bar{n} =n/H\)
\end{itemize}
Since these considerations are independent of one another, the complexity is dominated by the \(O(n^3)\) global SVD, which is a one-time cost. This compares favorably to other methods, which scale as in \Cref{tab:complexity}.
\begin{table}[!h]
\centering
\begin{tabular}{|c|c|c|c|}
\hline
\textbf{Method} & \textbf{Complexity} & \textbf{Dominant factor} & \textbf{Notes} \\
\hline\hline
Random & \(O(1)\) & None & Least performant \\
Leverage score & \(O(C\cdot n^3)\) & Cubic in dataset size & Theoretically most expensive \\
Greedy FPS & \(O(n_{\text{select}} \cdot n \cdot k \cdot C) \) & Quadratic in selection size & Tradeoff in \(n_{\text{select}}\) \\
K-Means & \(O(n_\text{init} \cdot t \cdot n \cdot n_{\text{select}} \cdot k \cdot C)\) & Linear in all factors & Tradeoff in \(n_{\text{select}}\); practically expensive \\
\textbf{Synthetic} & \(O(n^3 + m_\text{gap} \cdot n \cdot k \cdot C)\) & Cubic for one-time SVD & No complexity tradeoff \\
\hline
\end{tabular}
\caption{Complexity profiles of several sampling methods.}
\label{tab:complexity}
\end{table}

The number of synthetic gradients is given by
\[
m = \underbrace{\sum_{h=1}^{H} r_h}_{\text{local}} + \underbrace{|\cG|}_{\text{gap}} - \underbrace{m_{\text{redundant}}}_{\text{orthogonalization}}
\]

In practice, the number of gradients synthesized depends on the kernel structure (larger gaps will require more gradients), number of clusters provided, and hyperparameters \(\tau_v\) and \(\tau_g\).

In Fig. \ref{fig:grid-search}, we run a grid search on fidelity, accuracy, and number \(m\) of synthetic gradients as a function of \(\tau_v\) and \(\tau_g\). We also show the count of local and global distilled gradients as a function of \(\tau_v\) and \(\tau_g\), seeing that results on the optimal Pareto frontier are sporadically distributed across configurations.

\begin{figure}[!h]
    \centering
    \begin{subfigure}{0.75\linewidth}
        \centering
        \includegraphics[width=\linewidth]{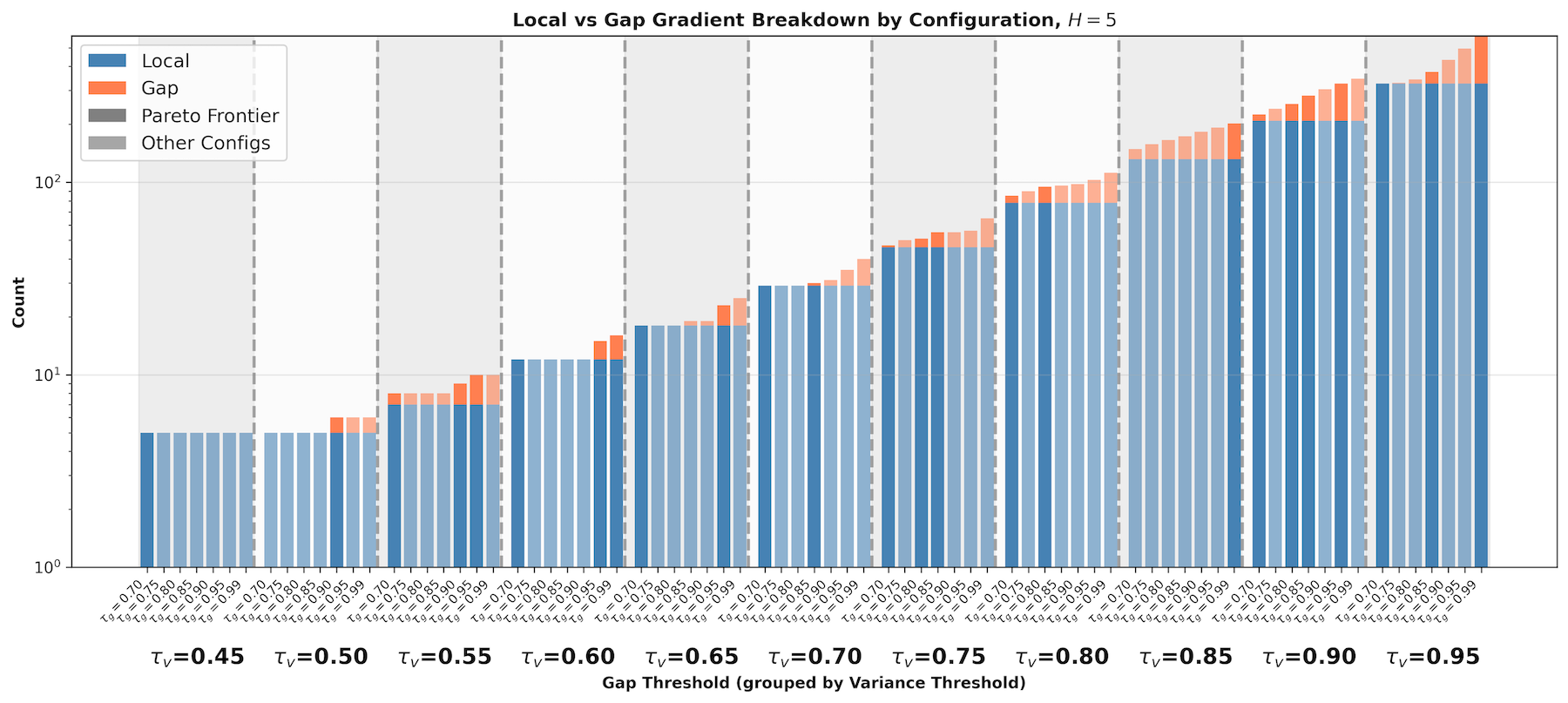}
        \includegraphics[width=\linewidth]{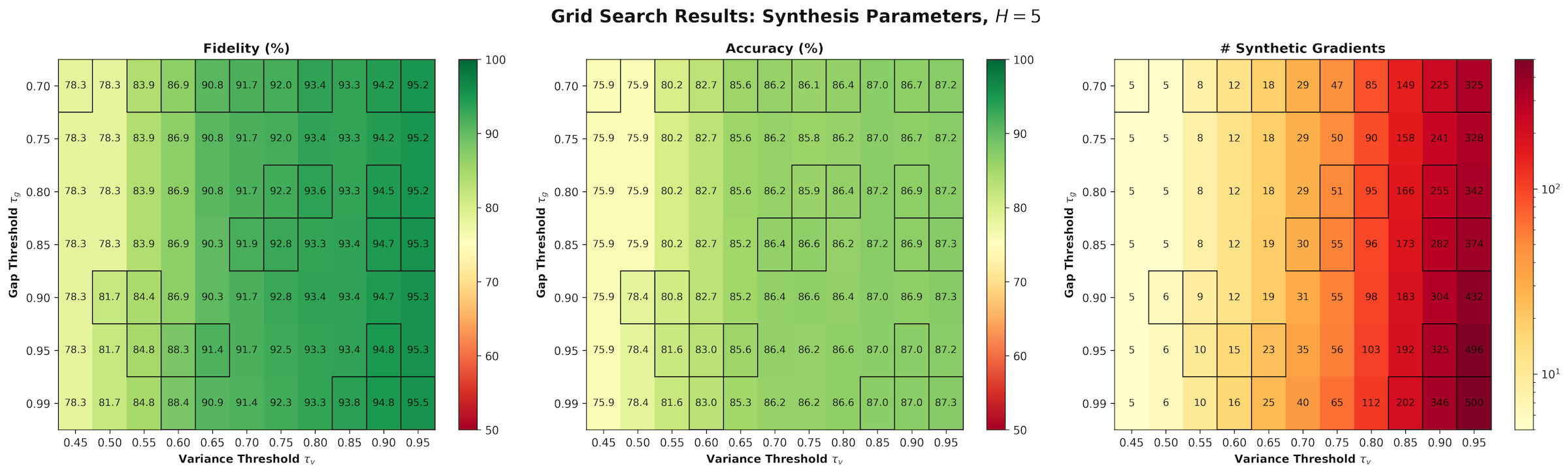}
        \caption{\(H=5\) clusters}
    \end{subfigure}
    \begin{subfigure}{0.75\linewidth}
        \centering
        \includegraphics[width=\linewidth]{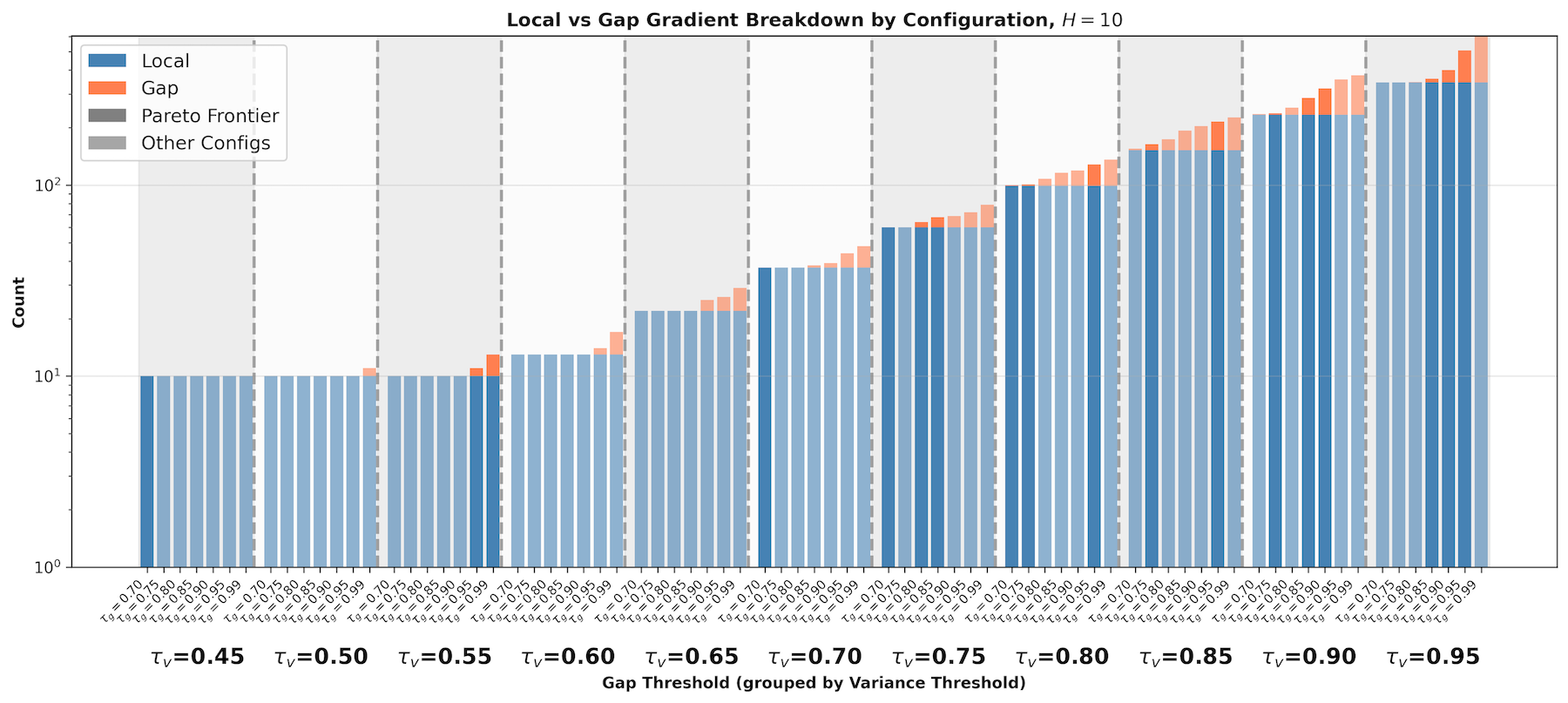}
        \includegraphics[width=\linewidth]{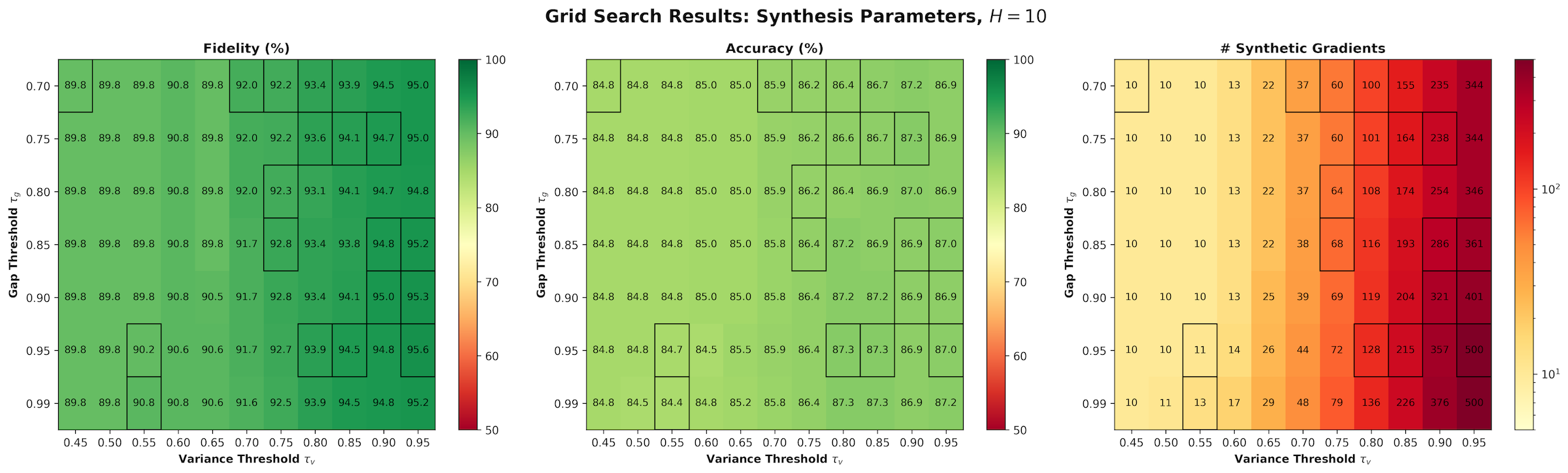}
        \caption{\(H=10\) clusters}
    \end{subfigure}
    \caption{Grid search results for \(H=5,10,15,20\) clusters on the ImageNette dataset (ResNet-18 model). 
    \textbf{Top}: Breakdown of the number of global \& local gradients synthesized by the algorithm for each configuration. Fewer clusters means that more effective ranks lie in the gaps. Configurations along the Pareto frontier (between \(\tau_v\) and \(\tau_g)\) are bolded.
    \textbf{Bottom}: Grid search between \(\tau_v\) and \(\tau_g\). Naturally, performance increases with the number of synthetic gradients. Configurations on the Pareto frontier are boxed.
    }
    \label{fig:grid-search}
\end{figure}

\clearpage  

\begin{figure}[!h]
    \ContinuedFloat
    \centering
    \begin{subfigure}[t]{0.75\linewidth}
        \centering
        \includegraphics[width=\linewidth]{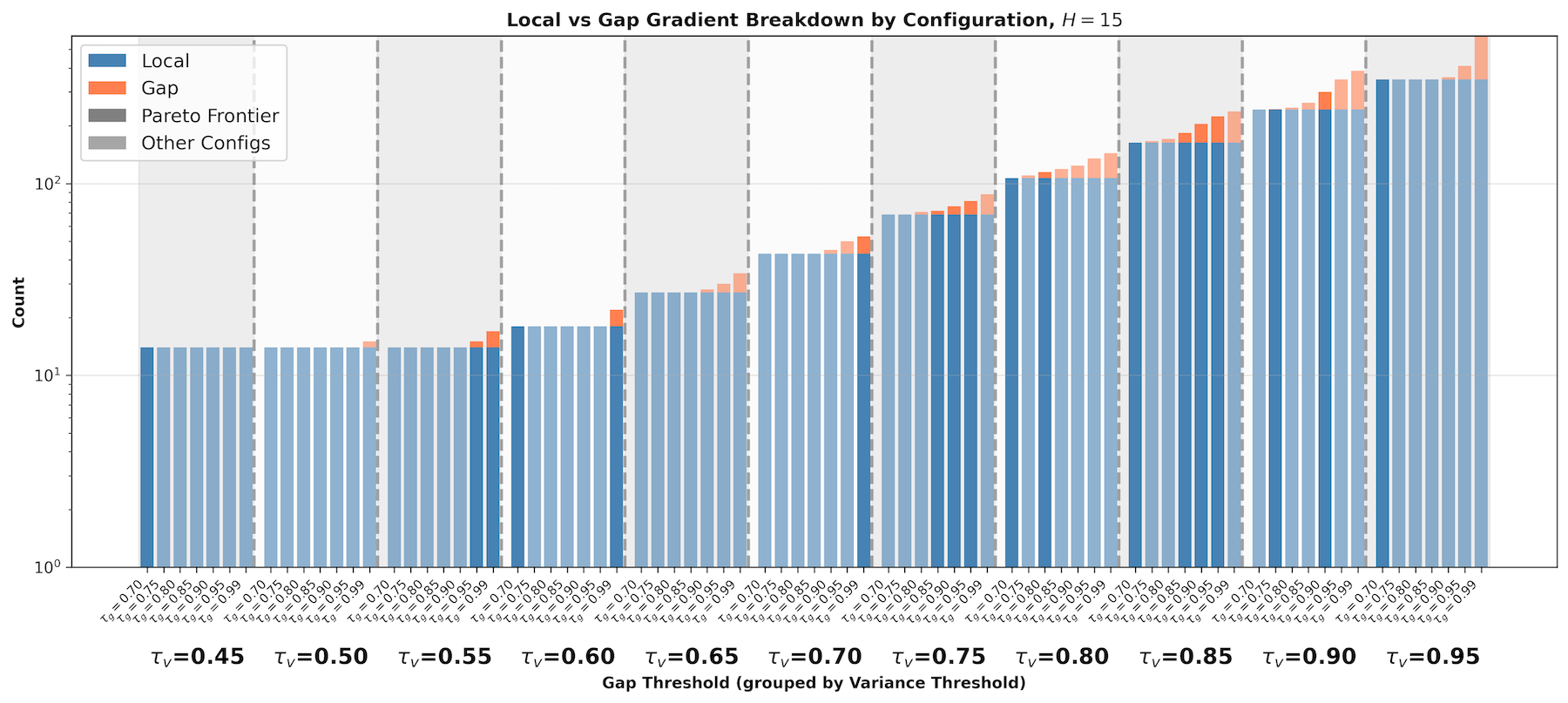}
        \includegraphics[width=\linewidth]{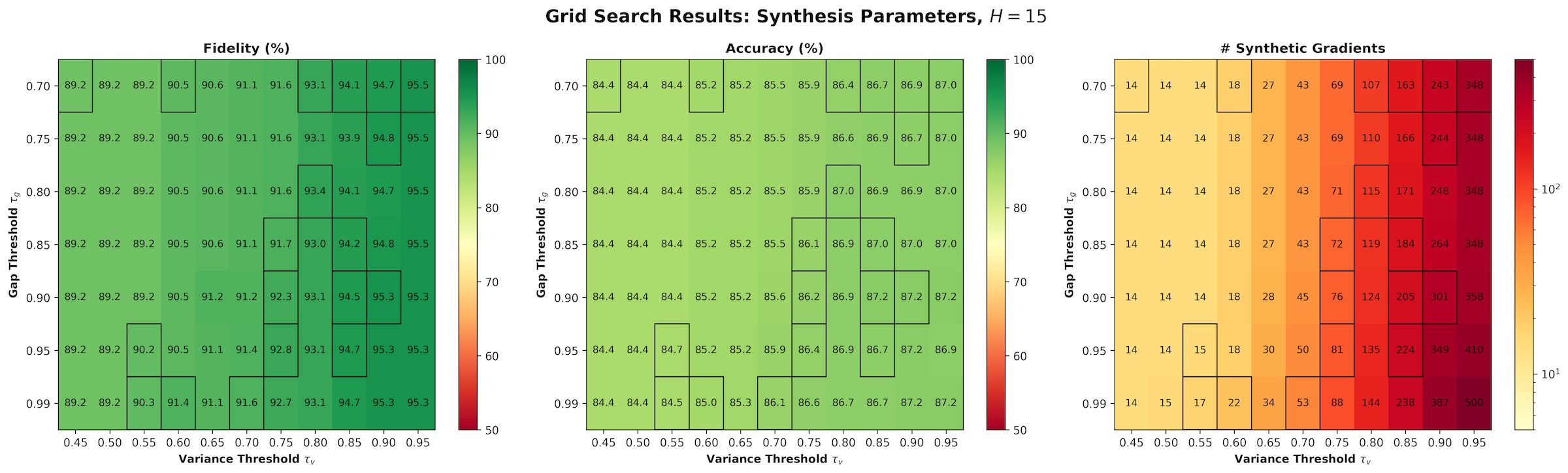}
        \caption{\(H=15\) clusters}
    \end{subfigure}
    \begin{subfigure}[t]{0.75\linewidth}
        \centering
        \includegraphics[width=\linewidth]{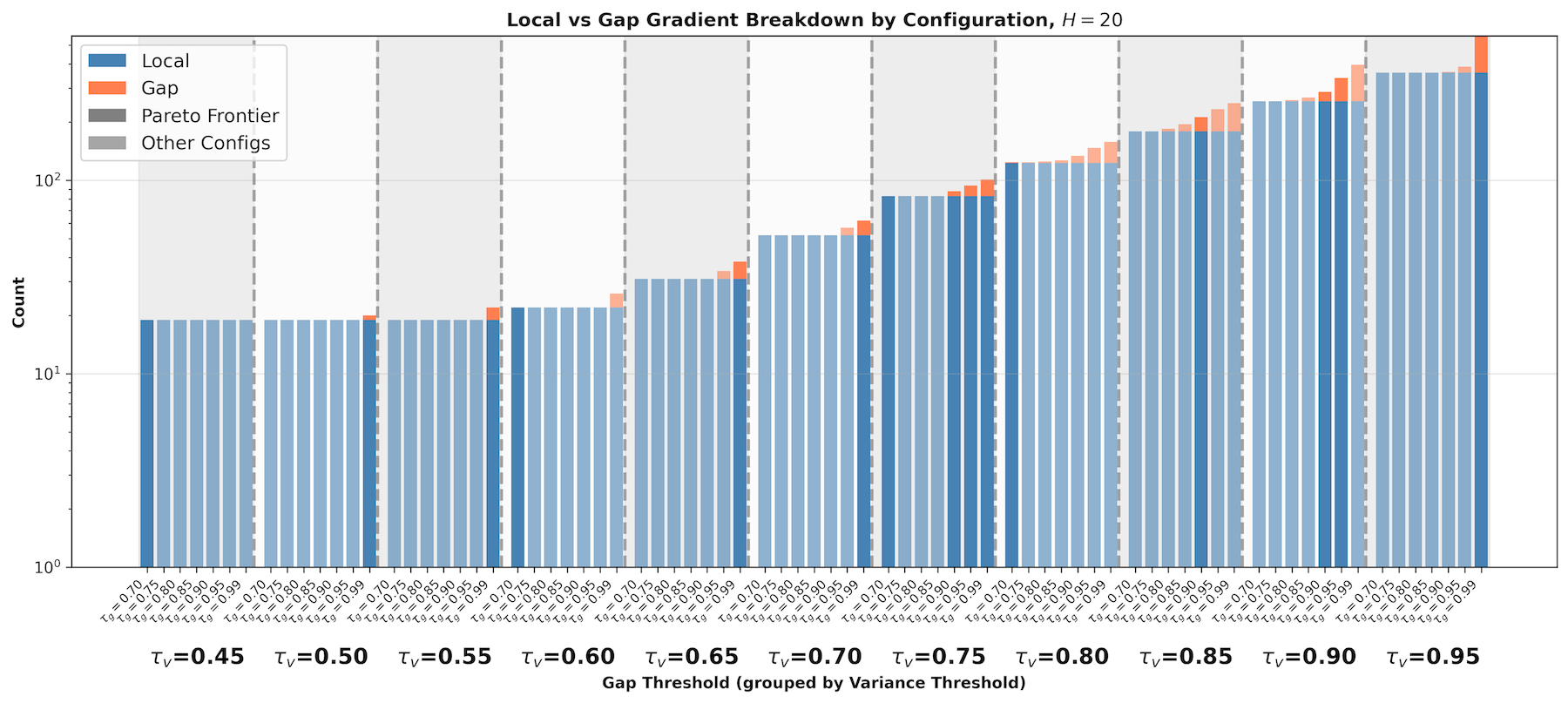}
        \includegraphics[width=\linewidth]{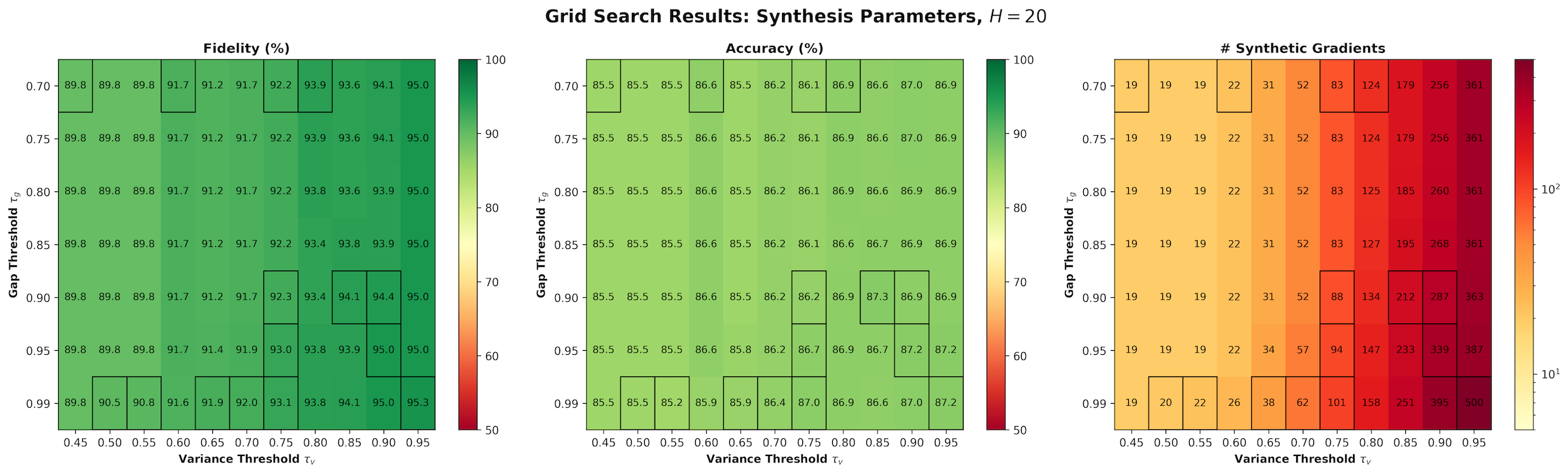}
        \caption{\(H=20\) clusters}
    \end{subfigure}
    \caption[]{(continued)}
\end{figure}

\end{document}